\documentclass[final, 12pt]{alt2023} % Include author names
\PassOptionsToPackage{linktocpage}{hyperref}
\usepackage{parskip} % Package to tweak paragraph skipping
\usepackage{tikz} % Package for drawing
\usepackage{amsmath}
\usepackage{amsfonts}
\usepackage{hyperref}
\usepackage{multirow}
\usepackage{algorithm}
\usepackage{algpseudocode}
\usepackage{csquotes}
\usepackage{algorithm}
\usetikzlibrary{trees}

        \hypersetup{
           breaklinks=true,   % splits links across lines
           colorlinks=true,   % displays links as colored text instead of blocks
           pdfusetitle=true,  % \title and \author values into pdf metadata
        }

\title[Expected Worst Case Regret via Stochastic Sequential Covering]{Expected Worst Case Regret via Stochastic Sequential Covering}
\usepackage{times}

\altauthor{%
 \Name{Changlong Wu} \Email{wuchangl@hawaii.edu}\\
 \Name{Mohsen Heidari} \Email{mheidari@purdue.edu}\\
 \Name{Ananth Grama} \Email{ayg@cs.purdue.edu} \\
 \Name{Wojciech Szpankowski} \Email{szpan@purdue.edu}\\
 \addr CSoI, Purdue University
}

\begin{document}

\newtheorem{claim}{Claim}
\newtheorem{problem}{Problem}
\newtheorem{fact}{Fact}

\newcommand{\x}{\textbf{x}}
\newcommand{\w}{\textbf{w}}
\newcommand{\z}{\textbf{z}}
\newcommand{\s}{\textbf{s}}
\newcommand{\y}{\textbf{y}}

\newcommand{\ayg}[1]{\textcolor{red}{{#1}}}

\def\CH{\mathcal{H}}
\def\VC{\mathsf{VC}}
\def\s{\mathsf{S}}
\def\SL{\mathsf{SL}}
\def\bs{\mathbf{s}}

\newcommand{\bt}[1]{\textbf{#1}}

\maketitle

\begin{abstract}
We study the problem of sequential prediction and online minimax regret with stochastically generated features under a general loss function. We introduce a notion of \emph{expected worst case minimax regret} that generalizes and encompasses prior known minimax regrets. For such minimax regrets we establish tight upper bounds via a novel concept of \emph{stochastic global sequential covering}. We show that for a hypothesis class of VC-dimension $\VC$ and $i.i.d.$ generated features of length $T$, the cardinality of the stochastic global sequential covering can be upper bounded with high probability (whp) by $e^{O(\VC \cdot \log^2 T)}$. We then improve this bound by introducing a new complexity measure called the \emph{Star-Littlestone} dimension, and show that classes with Star-Littlestone dimension $\SL$ admit a stochastic global sequential covering of order $e^{O(\SL \cdot \log T)}$. We further establish upper bounds for real valued classes with  finite fat-shattering numbers. Finally, by applying information-theoretic tools of the fixed design minimax regrets, we provide lower bounds for the  expected worst case minimax regret.  We demonstrate the effectiveness of our approach by establishing tight bounds on the expected worst case minimax regrets for logarithmic loss and 
general mixable losses.  

\end{abstract}

\begin{keywords}%
  online learning, stochastic sequential covering, minimax regret
\end{keywords}

\section{Introduction}
\label{sec-intro}

Online learning \citep{shalev2014understanding}  can be viewed as a game between Nature and 
predictor/ learner. At each time step $t$, 
Nature selects feature $\x_t\in \mathcal{X}$ and presents it to the learner. 
The learner then makes a prediction $\hat{y}_t\in\hat{\mathcal{Y}}$ based on 
history $\x^t=\{\x_1,\cdots,\x_t\}$ and $y^{t-1}=\{y_1,\cdots,y_{t-1}\}$, where 
$\x^t$ and $y^{t-1}$ are the feature vector up to time $t$, and the true labels up to time $t-1$, respectively. 
Nature then reveals the true label $y_t\in \mathcal{Y}$ and the learner incurs some loss.
The game continues up to time $T$ and the goal is to minimize 
the pointwise regret defined as:
$$
R(\hat{y}^T,y^T,\mathcal{H}\mid \x^T)=
\sum_{t=1}^T\ell(\hat{y}_t,y_t)-\inf_{h\in \mathcal{H}}\sum_{t=1}^T\ell(h(\x_t),y_t),
$$
where $\ell:\hat{\mathcal{Y}}\times\mathcal{Y}\rightarrow \mathbb{R}^+$ is a
loss function and $\mathcal{H}\subset \hat{\mathcal{Y}}^{\mathcal{X}}$ is a class of 
\emph{experts}. We say sequences $\x^T$ and $y^T$ are \emph{realizable} 
if for some $h\in \mathcal{H}$ we have 
$h(\x_t)=y_t$ for $t\in [T]$, and \emph{agnostic} otherwise. 

Regret analysis for a general class $\CH$ of experts is often 
studied via a sequential cover $\mathcal{G}$ of $\mathcal{H}$
which is defined as a set of functions mapping $\mathcal{X}^*\rightarrow \hat{\mathcal{Y}}$ 
(where $\mathcal{X}^*$ is the set of all finite sequences over $\mathcal{X}$) such that for all $h\in \mathcal{H}$ and $\x^T\in \mathcal{X}^T$ there exists $g\in \mathcal{G}$ 
satisfying $h(\x_t)=g(\x^t)$ for all $t\in [T]$. In the seminal work of ~\cite{ben2009agnostic}, the authors 
established a striking connection between the regrets of agnostic and realizable 
cases through the concept of sequential covering. One of the core arguments 
of~\cite{ben2009agnostic} is the following observation (informally):
\begin{displayquote}
    \emph{If a binary valued class $\mathcal{H}$ admits a predictor with  cumulative 
error upper bounded by $B$ for {realizable} samples of length $T$, 
then $\mathcal{H}$ has a {sequential cover} $\mathcal{G}$ of size $O(T^{B+1})$.}
\end{displayquote}
This is established by considering all the possible error patterns of the predictor 
in the realizable case (see Lemma~\ref{lem2}). 
Using standard expert algorithms (e.g., Exponential Weighted Average), one then
relates the \emph{agnostic} regrets to the size of $\mathcal{G}$. For example, for 
absolute loss, the regret bound is of the form 
$O(\sqrt{T\log|\mathcal{G}|})=O(\sqrt{BT\log T})$, while for general mixable losses one  finds
$O(\log|\mathcal{G}|)=O(B\log T)$. In~\citep{ben2009agnostic}, the authors 
derived upper bounds for $B$ through the Littlestone dimension. This was further 
generalized in~\citep{daniely2011multiclass} to multi-label cases, and 
in~\citep{rakhlin2010online} to the real valued case. However, all of these 
results assumed that features are presented adversarially. 
This may be too pessimistic and  restrictive for modeling real scenarios of prediction, see e.g.~\citep{rakhlin2011online1}.

This paper generalizes the concept of \emph{sequential covering} 
to a stochastic scenario. Instead of assuming that $\x^T$ is presented adversarially, 
we postulate that $\x^T$ is sampled from some \emph{unknown} distribution $\mu$ 
over $\mathcal{X}^T$ (not necessarily $i.i.d.$) in a known class $\mathcal{P}$ 
of distributions (i.e., the so called {\it universal} or {\it distribution blind} scenario).
However, the true labels are still presented adversarially. We say a class $\mathcal{G}$ 
of functions $\mathcal{X}^*\rightarrow \hat{\mathcal{Y}}$ is a \emph{stochastic 
(global) sequential covering} of $\mathcal{H}$ w.r.t. $\mathcal{P}$ and confidence 
$\delta>0$, if for all $\mu\in \mathcal{P}$, we have:
$$
\mathrm{Pr}_{\x^T\sim \mu}\left[\exists h\in \mathcal{H}~\forall 
g\in \mathcal{G}\exists t\in [T]~s.t.~h(\x_t)\not=g(\x^t)\right]\le \delta.
$$

In this paper we apply stochastic sequential covering to analyze minimax regrets.
We introduce a novel general \emph{expected worst case minimax regret} defined as:
$$
\tilde{r}_T(\mathcal{H},\mathcal{P})=\inf_{\phi^T}\sup_{\mu\in \mathcal{P}}
\mathbb{E}_{\x^T\sim \mu}\left[\sup_{y^T}R(\hat{y}^T,y^T,\mathcal{H}\mid \x^T)\right],
$$
where $\phi_t:\mathcal{X}^{t}\times\mathcal{Y}^{t-1}\rightarrow 
\hat{\mathcal{Y}}$ runs over all possible prediction rules, and $\hat{y}_t=\phi_t(\x^t,y^{t-1})$.
The expected worst case minimax regret recovers previously known minimax regrets by considering different classes of $\mathcal{P}$ (see Proposition~\ref{prop1}).

Similar to the adversary case, we will show that the expected worst case minimax regret can be upper bounded by the size of a stochastic (global) sequential covering set 
$\mathcal{G}$ with confidence $\delta=\frac{1}{T}$ via standard expert algorithms. 
For example, for absolute loss, we have 
$\tilde{r}_T(\mathcal{H},\mathcal{P})=O(\sqrt{T\log|\mathcal{G}|})$, while for logarithmic 
loss and general mixable loss (e.g., square loss), we find
$\tilde{r}_T(\mathcal{H},\mathcal{P})=O(\log|\mathcal{G}|)$; see Theorem~\ref{th1} and \ref{th2} 
for proofs.

Our goal is to derive tight regret bounds that go beyond the conventional 
$\sqrt{T}$ bounds for a general loss. To retain focus, we will mainly consider 
the case when $\mathcal{P}$ is the class of all $i.i.d.$ distributions over 
$\mathcal{X}^T$. We should emphasize that 
our results also work for general \emph{exchangeable} distributions~\citep{aldous1985exchangeability} and distributions with sufficient symmetry, as discussed in the Appendix~\ref{app-symmetry}.

\paragraph{Summary of results.}
Our first main \emph{technical} result is the following high probability cumulative error 
bound for the \emph{1-inclusion graph} algorithm~\citep{haussler1994predicting} 
in realizable case:

\textbf{Lemma~\ref{lem1}.}\textit{
Let $\Phi$ be the 1-inclusion graph prediction algorithm, $\mathcal{H}$ be a 
class with finite VC-dimension and $\mu$ be an arbitrary $i.i.d.$ distribution 
over $\mathcal{X}^T$. Then for all $\delta>0$ 
$$
\mathrm{Pr}_{\x^T\sim \mu}\left[\sup_{h\in \mathcal{H}}
\sum_{t=1}^T1\{\Phi(\x^t,h(\{\x^{t-1}\}))\not=h(\x_t)\}\ge 
O(\VC(\mathcal{H})\log T+\log(1/\delta))\right]\le \delta
$$
where $h(\{\x^{t-1}\})=\{h(\x_1),\cdots,h(\x_{t-1})\}$ and $O$ hides absolute 
constant independent of $\VC(\mathcal{H}),T,\delta$.}

Applying this lemma and the error 
pattern counting argument as in~\citep{ben2009agnostic} (see Lemma~\ref{lem2}), 
we arrive at our first main general result:

\textbf{Theorem~\ref{th4} (Informally).} \textit{Let $\mathcal{P}$ be the class 
of all $i.i.d.$ distributions over $\mathcal{X}^T$, and $\mathcal{H}$ 
be a binary valued class with finite VC-dimension. Then, there exists a 
stochastic sequential covering $\mathcal{G}$ of $\mathcal{H}$ w.r.t. $\mathcal{P}$ 
at confidence $\delta$ such that 
$$\log|\mathcal{G}|\le O(\VC (\mathcal{H})\log^2T+\log T\log(1/\delta)).$$
As a consequence,
for finite VC-dimension classes and logarithmic loss $\ell$,
minimax regret is upper bounded by
$\tilde{r}_T(\mathcal{H},\mathcal{P})\le O(\VC (\mathcal{H})\log^2T).$
}

This result 
improves substantially the $\sqrt{T}$ regret bounds of~\cite{bhatt2021sequential} 
and subsumes the bounds implied by~\cite{bilodeau2021minimax}
only proved for the weaker \emph{average} and \emph{realizable} case.

The next natural question is whether the $\log^2 T$ factor in Theorem~\ref{th4} can be 
improved to $\log T$ for general finite VC-dimensional classes, which matches our
lower bound discussed in Section~\ref{sec-lower} and Appendix~\ref{app-omega}.  While we are 
unable to answer this question in its generality, we can show that the answer 
is affirmative for many interesting non-trivial classes. To achieve this, we 
introduce a new complexity measure, called the \emph{Star-Littlestone} 
dimension $\SL$ (see Section~\ref{sec-binary}):

\textbf{Definition (Star-Littlestone Dimension).} \textit{Let $\mathcal{H}$ be a binary 
valued class. For any numbers $s,d>0$, we say $\mathcal{H}$ Star-Littlestone 
shatters an $\mathcal{X}$-valued full binary 
tree $\tau$ of depth $d$ with the star scale $s$ if a subclass of $\mathcal{H}$ consistent with 
any path of $\tau$ has \emph{Star number}~\citep{hanneke2015minimax} greater than $s$. 
The Star-Littlestone dimension $\SL(s)$ of $\mathcal{H}$ at star scale $s$ is defined to 
be the maximum number $d$ such that there exists tree $\tau$ of depth $d$ that can be 
Star-Littlestone shattered by $\mathcal{H}$.}

The Star-Littlestone dimension is a more general concept than 
\emph{Star number} (see Section~\ref{sec-binary}) and \emph{Littlestone dimension}. 
Indeed, let $\mathcal{H}=\{h_{[a,b]}(x)=1\{x\in [a,b]\}:[a,b]\subset [0,1]\}$.
Then $\mathcal{H}$ has both infinite 
Star number and Littlestone dimension, but Star-Littlestone dimension is $0$ 
at scale $4$. 

\textbf{Theorem~\ref{th6} (Informally).} \textit{Let $\mathcal{P}$ be the 
class of all $i.i.d.$ distributions over $\mathcal{X}^T$ and 
$\mathcal{H}$ be a binary valued class with Star-Littlestone dimension $\SL(s)$ 
at star scale $s$. Then, there exists a stochastic sequential 
covering set $\mathcal{G}$ of $\mathcal{H}$ w.r.t. $\mathcal{P}$ at confidence $\delta$ 
such that: 
$$\log|\mathcal{G}|\le O(\max\{\SL(s)+1,s\}\log T+\log(1/\delta)).$$
As a consequence, for logarithmic loss, we have:
$
\tilde{r}_T(\mathcal{H},\mathcal{P})\le O(\max\{\SL(s)+1,s\}\log T).
$}

\noindent The key to proving Theorem~\ref{th6} is an explicit construction of $\mathcal{G}$ via an important property of a finite Star number class established in Lemma~\ref{lem3}, which does not rely on error pattern counting as in~\citep{ben2009agnostic}. Our construction is new and may be of independent interest.

Our next main result bounds the stochastic sequential covering 
of real valued function classes:

\textbf{Theorem~\ref{th7} (Informally).} \textit{Let $\mathcal{P}$ be the class of 
all $i.i.d.$ distributions over $\mathcal{X}^T$ and $\mathcal{H}$ be a 
$[0,1]$-valued class with $\alpha$-fat shattering number of order $d(\alpha)$. 
Then $\mathcal{H}$ admits a stochastic sequential covering set 
$\mathcal{G}$ with scale $\alpha$ and confidence $\delta$ such that:
$$\log |\mathcal{G}|\le O(d(\alpha/32)(\log T\log 1/\alpha)^4+
\log T\log(\log T/\delta)).$$
In particular, for any class with $\alpha$-fat shattering 
number of order $\alpha^{-l}$ for some $l \ge 0$, the 
expected worst case minimax regret bound is 
$\tilde{O}(T^{l/(l+1)})$ under logarithmic loss. Moreover, this bound can not 
be improved upto poly-logarithmic factors in general. 
}

We complete this introduction with the following lower bounds discussed in Section~\ref{sec-lower}
that relate the expected worst case minimax regret to the fixed design 
minimax regret $r_T^*(\CH| \x^T)$ defined in (\ref{eq-rHx}).

\textbf{Theorem~\ref{th8} (Informally).} \textit{ Let $\mathcal{P}$ be 
the class of all $i.i.d.$ distributions over 
$\mathcal{X}^T$ and $\mathcal{H}$ be any $[0,1]$-valued class. 
If the fixed design regret $r^*_T(\mathcal{H}\mid \x^T)$ satisfies minor regularity conditions, then:
$$\tilde{r}_T(\mathcal{H},\mathcal{P})\ge (1-O(1/\log T)) 
r^*_{(T/\log T)}(\mathcal{H}),$$
where $r_T^*(\mathcal{H})=\sup_{\x^T} r_T^*(\CH | \x^T)$.
In particular, there exists a class $\mathcal{H}$ of finite VC-dimension, 
such that if $\ell$ is the log-loss and $\mathcal{P}$ is the class of 
$i.i.d.$ distributions over $\mathcal{X}^T$, then:
$$\tilde{r}_T(\mathcal{H},\mathcal{P})\ge (1-O(1/\log T))\VC (\mathcal{H})
\log(T/\VC(\mathcal{H})).$$
}

Theorem~\ref{th8} implies that any lower bounds established for the fixed 
design regrets as in \citep{wu2022precise} can be translated to regrets for the expected worst case minimax regret that achieves the same leading constant for $\text{ploy}\log T$ regrets 
and looses only a $\log T$ factor for polynomial regrets.

\paragraph{Related work.}
Regret analysis of online learning problems dates back to the work 
of~\cite{littlestone1994weighted} and~\cite{vovk90}, where the authors 
developed a general framework for the \emph{Exponential Weighted Average} 
algorithm for finite expert classes. We refer to~\citep{lugosi-book} for an 
excellent discussion of this topic and its extensions. 
In~\citep{ben2009agnostic}, the authors extended the framework to 
infinite classes with binary labels via the concept of sequential covering
and latter generalized to the multi-class case in~\citep{daniely2011multiclass}. 
In a series of 
papers~\citep{rakhlin2010online,rakhlin2015sequential,rakhlin2015martingale,rakhlin14}, 
the authors established a comprehensive framework for  regret analysis of 
real valued classes via the concept of~\emph{sequential Rademacher complexity}. 
One of the core arguments of this line of work is to express  regret in 
terms of an iterated minimax formulation, which will then be controlled by an 
expected majorizing of martingales via the minimax theorem. The latter is then controlled by the \emph{sequential covering}~\footnote{Note that the sequential covering as in~\citep{rakhlin2010online} is slightly different than 
the one we discussed in our paper, since their definition relies on some underlying trees.} 
number and the standard technique of chaining and Dudley integral. However, all of 
these efforts consider adversarial cases,
which can be too restrictive for real word scenarios.

In~\citep{lazaric2009hybrid}, the authors introduced a scenario where the 
features are generated by an unknown $i.i.d.$ source but the labels are still presented adversarially. 
In particular,~\cite{lazaric2009hybrid} showed that for finite VC-dimensional 
classes and for absolute loss, regret grows as $O(\sqrt{\VC (\mathcal{H})T\log T})$. 
One of the core arguments of this work is an epoch approach that reduces infinite 
class to the finite class case using successive covering. However, 
their upper bound is dominated by a $\sqrt{T}$ term of the approximation error 
of covering, which may be too loose for many loss functions, e.g., logarithmic loss. 
Indeed, the same epoch approach (and its analysis of the approximation error) 
was used in~\citep{bhatt2021sequential} for logarithmic loss, resulting an 
$O(\sqrt{T})$ regret bound. In~\citep{bilodeau2021minimax}, the authors showed 
that for logarithmic loss and finite VC-dimensional classes, regret grows 
as $O(\VC (\mathcal{H})\log^2 T)$. However, their proof  applies only for the
\emph{average case} minimax regret (see Section~\ref{sec-formulation}) 
and in the \emph{realizable} (i.e., \emph{well-specified}) case. 
In~\citep{rakhlin2011online1}, the authors considered a scenario where at each time step Nature selects adversarially some distribution to sample from a restricted class of distributions that are determined by previously generated samples (\emph{not} precisely selected distributions). This is characterized by the concept of \emph{distribution dependent} Rademacher complexity, using a similar minimax approach as discussed above. However, their result only holds for the \emph{distribution non-blind case} (i.e., when the distribution is known in advance), see~\cite[Section 7]{rakhlin2011online1}. Note that all the regrets analyzed in this paper are for \emph{distribution blind} case. We note also a recent line of research of 
\emph{smooth adversaries} in~\citep{rakhlin2011online1,haghtalab2020smoothed,haghtalab2022smoothed,block2022smoothed}, which shares some technical similarity (e.g., symmetries of samples) with our work, see Appendix~\ref{app-nonblind}.

\section{Problem Formulation}
\label{sec-formulation}

Let $\mathcal{X}$ be a feature space, $\hat{\mathcal{Y}}$ be 
the output space, and $\mathcal{Y}$ be the true label space. 
We denote by $\mathcal{H}\subset \hat{\mathcal{Y}}^{\mathcal{X}}$ 
a class of functions $\mathcal{X}\rightarrow \hat{\mathcal{Y}}$, 
which is also referred as a hypothesis or experts class. For any time horizon $T$, 
we consider a class $\mathcal{P}$ of distributions over $\mathcal{X}^T$. 
We consider the following game between Nature and predictor that proceeds 
as follows. At the beginning of the game, Nature 
selects a distribution $\mu \in \mathcal{P}$ and samples an input 
sequence $\x^T\sim \mu$, where $\x^T\in \mathcal{X}^T$. At each time step $t\le T$, 
Nature reveals the $t$-th sample $\x_t$ of $\x^T$ to the predictor. 
The predictor then makes a prediction $\hat{y}_t\in \mathcal{Y}$ using a strategy 
$\phi_t: \mathcal{X}^{t}\times \mathcal{Y}^{t-1}\rightarrow \hat{\mathcal{Y}}$ 
potentially using the history observed thus far, that is,
$\hat{y}_t=\phi_t(\x^t,y^{t-1})$. After the prediction, Nature reveals 
the true labels $y_t$ and the predictor incurs a loss 
$\ell(\hat{y}_t,y_t)$ for some predefined loss function 
$\ell:\hat{\mathcal{Y}}\times \mathcal{Y}\rightarrow [0,\infty)$. 
We are interested in the following \emph{expected worst case} minimax regret
\begin{equation}
\label{eq-rtilde}
\tilde{r}_T(\mathcal{H},\mathcal{P})=\inf_{\phi^T}\sup_{\mu\in \mathcal{P}}
\mathbb{E}_{\x^T\sim \mu}\left[\sup_{y^T}\left(\sum_{t=1}^T\ell(\hat{y}_t,y_t)-
\inf_{h\in \mathcal{H}}\sum_{t=1}^T\ell(h(\x_t),y_t)\right)\right],
\end{equation}
where \emph{worst case} indicates the predictor needs to 
compete  with the best expert in $\mathcal{H}$ for any $\x^T$. 

We note that the expected worst case minimax regret $\tilde{r}_T(\mathcal{H},\mathcal{P})$ recovers previously known minimax regrets by selecting appropriate distribution class $\mathcal{P}$.
Indeed, in~\citep{ss21,jss20,wu-isit22}, the following regrets are defined. 
The {\it fixed design} minimax regret for any given $\x^T\in \mathcal{X}^T$ is defined as: 
\begin{equation}
\label{eq-rHx}
r^*_T(\mathcal{H}\mid \x^T)=\inf_{\phi^T}\sup_{y^T}\left(\sum_{t=1}^T
\ell(\hat{y}_t,y_t)-\inf_{h\in \mathcal{H}}\sum_{t=1}^T\ell(h(\x_t),y_t)\right).
\end{equation}
The {\it maximum} fixed design minimax regret is then: $r^*_T(\mathcal{H})=\sup_{\x^T}r^*_T(\mathcal{H}\mid \x^T)$. Furthermore, the {\it sequential minimax regret} is
\begin{equation}
\label{eq-ra}
r^a_T(\mathcal{H})=\inf_{\phi^T}\sup_{\x^T,y^T}\left(\sum_{t=1}^T\ell(\hat{y}_t,y_t)
-\inf_{h\in \mathcal{H}}\sum_{t=1}^T\ell(h(\x_t),y_t)\right)
\end{equation}
which is equivalent\footnote{The equivalence, we believe, is a folklore result, see 
e.g.,~\cite[Exercise 2.18]{lugosi-book} or~\cite[Lemma~2]{wu-isit22} for a proof.} to the iterated minimax regret as in~\citep{rakhlin2010online}.

We also introduce the following expected \emph{average case} minimax regret: 
\begin{equation}
\label{eq-rbar}
\bar{r}_T(\mathcal{H},\mathcal{P})=\inf_{\phi^T}\sup_{\mu\in \mathcal{P},h\in 
\mathcal {H}}\mathbb{E}_{\x^T\sim \mu}\left[\sup_{y^T}\left(\sum_{t=1}^T
\ell(\hat{y}_t,y_t)-\ell( h(\x_t),y_t)\right)\right]
\end{equation}
where the main difference with $\tilde{r}_T(\mathcal{H},\mathcal{P})$  is the position of $\sup_h$.
Note that this concept subsumes the setups 
of~\citep{bhatt2021sequential,bilodeau2021minimax} except that the authors of
\citep{bhatt2021sequential,bilodeau2021minimax} 
consider a weaker \emph{realizable} setting for generating $y^T$.

The following observation is easy to prove and shows that $\tilde{r}_T$ is indeed a 
more general concept:

\begin{proposition}
\label{prop1}
If $\mathcal{P}$ is a class of all singleton distributions over $\mathcal{X}^T$, then
$\tilde{r}_T(\mathcal{H},\mathcal{P})=r^a_T(\mathcal{H})$
for all $\mathcal{H}$. If $\mathcal{P}$ is the singleton 
distribution that assigns probability $1$ for $\x^T$, then
$\tilde{r}_T(\mathcal{H},\mathcal{P})=r^*_T(\mathcal{H}\mid \x^T).$
Furthermore, 
$\tilde{r}_T(\mathcal{H},\mathcal{P})\ge \bar{r}_T(\mathcal{H},\mathcal{P})$,
for any $\mathcal{H}$ and $\mathcal{P}$. 
\end{proposition}

\begin{example}
To understand the differences between $\tilde{r}_T$ and $\bar{r}_T$, 
we consider the following example. Let $\mathcal{H}$ be the class of all 
functions $[0,1]\rightarrow \{0,1\}$ that takes value $1$ on at most 
$T$ positions and $0$ otherwise. Let $\nu$ be the uniform distribution over $[0,1]$, 
and $\ell(\hat{y}_t,y_t)=|\hat{y}_t-y_t|$, where $\hat{y}_t\in [0,1]$ 
and $y_t\in \{0,1\}$. We will denote by $\nu^T$ the $i.i.d$ distribution of 
length $T$ with marginal $\nu$. We have $\bar{r}_T(\mathcal{H},\{\nu^T\})=0$, since for any $h$, w.p. $1$ we have $h(x_t)=0$ for all $t\in [T]$, meaning 
that a strategy that predicts $0$ all the time incurs $0$ regret. However, we also have $\tilde{r}_T(\mathcal{H},\{\nu^T\})\ge \frac{T}{2}.$ To see this, we choose $y^T\in \{0,1\}^T$ uniformly at random and observe that any strategy will make at least $\frac{T}{2}$ accumulated losses, 
however, for any $\x^T$ and $y^T$, there exists $h\in \mathcal{H}$ 
such that $\forall t\in [T],~h(\x_t)=y_t$.

\end{example}

We should remark that our definition of both $\tilde{r}_T$ and $\bar{r}_T$ are \emph{distribution blind} in the sense of~\cite[Section 7]{rakhlin2011online1}, since the marginals of $\mu$ can be dependent arbitrarily, not just through previously generated samples as in~\citep{rakhlin2011online1}.

\section{Upper Bounds on Regret via Stochastic Sequential Cover}
\label{sec-upper}

This is the main section of our paper, where we provide general upper bounds
for the expected worst case minimax regret (see Theorems~\ref{th1} and \ref{th2})
via the novel concept of \emph{stochastic global sequential covering}. 
Without loss of generality, 
we  assume that $\hat{\mathcal{Y}}=[0,1]$ in the sequel.
 
\begin{definition}
\label{def-gcover}
We say a class $\mathcal{G}$ of functions $\mathcal{X}^*\rightarrow [0,1]$ 
is stochastic global sequential cover of a class $\mathcal{H}\subset [0,1]^{\mathcal{X}}$ 
w.r.t. the class $\mathcal{P}$ of distributions over $\mathcal{X}^T$ 
at scale $\alpha>0$ and confidence $\delta>0$, if for all $\mu\in \mathcal{P}$, 
we have
$$\mathrm{Pr}_{\x^T\sim \mu}\left[\exists h\in \mathcal{H}~\forall 
g\in \mathcal{G}~\exists~t\in [T]~s.t.~|h(\x_t)-g(\x^t)|>\alpha\right]\le \delta.$$
We define the minimal size of $\mathcal{G}$ to be the stochastic 
global sequential covering number of $\mathcal{H}$.
\end{definition}

Before we provide effective bounds on the stochastic global sequential 
covering number, we first demonstrate how a bound on the covering 
number implies bounds on the regret $\tilde{r}_T$.

\begin{theorem}
\label{th1}
Let $\ell(\cdot,y)$ be convex, $L$-Lipschitz and bounded by $1$ on 
$\hat{\mathcal{Y}}$ for any $y\in \mathcal{Y}$ and $\mathcal{H}$ be a set of functions 
$\mathcal{X}\rightarrow [0,1]$. Let $\mathcal{G}_{\alpha}$ be 
a stochastic global sequential covering of $\mathcal{H}$ at 
scale $\alpha$ and confidence $\delta=1/T$ w.r.t. distribution class $\mathcal{P}$. Then
$$\tilde{r}_T(\mathcal{H},\mathcal{P})\le \inf_{0\le \alpha\le 1}
\left\{\alpha LT+\sqrt{(T/2)\log|\mathcal{G}_{\alpha}|}+1\right\}.$$
If, in addition, $\ell$ is $\eta$-Mixable~\cite[Chapter 3.5]{lugosi-book} then
$$
\tilde{r}_T(\mathcal{H},\mathcal{P})\le \inf_{0\le \alpha \le 1}
\left\{\alpha LT+\frac{1}{\eta}\log|\mathcal{G}_{\alpha}|+1\right\}.
$$
\end{theorem}
\begin{proof}
We run the Exponential Weighted Average (EWA) algorithm~\cite[Page 14]{lugosi-book} 
on $\mathcal{G}_{\alpha}$. We split the regret into two parts, one that is 
incurred by the predictor against $\mathcal{G}_{\alpha}$  and the other that is
incurred by the discrepancy between $\mathcal{G}_{\alpha}$ and $\mathcal{H}$. 
For the first term, we have by standard result~\cite[Theorem 2.2]{lugosi-book} 
that with probability $1$ on $\x^T$:
$$\sum_{t=1}^T\ell(\hat{y}_t,y_t)\le \inf_{g\in \mathcal{G}_{\alpha}}
\sum_{t=1}^T\ell(g(\x^t),y_t)+\sqrt{(T/2)\log|\mathcal{G}_{\alpha}|}.$$
For the second term, we denote by  $A$ the event described in the probability of 
Definition~\ref{def-gcover}. Since $\mathrm{Pr}[A]\le \frac{1}{T}$ and $\ell(\hat{y},y)\le 1$ 
by assumption, the expected regret contributed by the event $A$ 
is at most $1$. We now condition on the event that $A$ does not happen. 
By Definition~\ref{def-gcover}, we obtain
$\forall h\in \mathcal{H}\exists g\in \mathcal{G}_{\alpha}\forall 
t\in [T],~|h(\x_t)-g(\x^t)|\le\alpha.$
Since $\ell$ is bounded by $1$ and $L$-Lipschitz, we have: 
$$
\inf_{h\in \mathcal{H}}\sum_{t=1}^{T}\ell(h(\x_t),y_t)\ge \inf_{g\in \mathcal{G}_{\alpha}}
\sum_{t=1}^T\ell(g(\x^t),y_t)-\alpha LT.
$$
The result follows by combining these inequalities. The last part follows by 
replacing the EWA algorithm with the Aggregation Algorithm as 
in~\cite[Chapter 3.5]{lugosi-book} and applying the result in~\cite[Proposition 3.2]{lugosi-book}.
\end{proof}
 
 \begin{theorem}
\label{th2}
Let $\ell$ be the logarithmic loss $\ell(\hat{y},y)=-y\log(\hat{y})-(1-y)\log(1-\hat{y})$ and $\mathcal{H}$, $\mathcal{G}_{\alpha}$, $\mathcal{P}$ be as in 
Theorem~\ref{th1}, then:
$$
\tilde{r}_T(\mathcal{H},\mathcal{P})\le \inf_{0\le \alpha\le 1}\left\{2\alpha T +
\log(|\mathcal{G}_{\alpha}|+1)+\log(|\mathcal{G}_{\alpha}|+1)/T+1\right\}.$$
\end{theorem}
\begin{proof}
The proof is similar to the proof of Theorem~\ref{th1}, but replacing the 
Exponential Weighted Average algorithm with the Smooth truncated Bayesian Algorithm introduced recently in~\citep{wu2022precise} and running 
the algorithm on $\mathcal{G}_{\alpha}\cup \{g\}$ with truncation parameter 
$\alpha$ and uniform prior, where $g$ is the function that maps to 
$\frac{1}{2}$ for all $\x^t$. We again split the regret into two parts, 
one incurred by the Smooth truncated Bayesian Algorithm, and the other 
incurred by the error of covering. By~\cite[Theorem 1]{wu2022precise}, 
the first term is upper bounded by
$2\alpha T +\log(|\mathcal{G}_{\alpha}|+1).$
For the error term, we note that we have added the all $\frac{1}{2}$ 
valued function $g$  into the expert class when running the Smooth truncated Bayesian Algorithm. This implies that the prediction rule can only incur the 
\emph{actual} accumulated losses upper bounded by $T+\log(|\mathcal{G}_{\alpha}|+1)$. 
Therefore, when the event $A$ (defined in  Theorem~\ref{th1}) happens, 
the expected regret only contributes
$(T+\log(|\mathcal{G}_{\alpha}|+1)) \cdot \mathrm{Pr}[A]\le (T+\log(|\mathcal{G}_{\alpha}|+1))/T.$
The result follows by combining the inequalities.
\end{proof}

\subsection{Stochastic Cover for Binary Valued Class with Finite VC-dimension}
\label{sec-binary}

Now, we focus on bounding the cardinality of the stochastic global cover.
We assume that $\mathcal{P}$ is the class of all $i.i.d.$ distributions
over $\mathcal{X}^T$; however, our results
hold for \emph{exchangeable} processes~\citep{aldous1985exchangeability} over $\mathcal{X}^T$ as well, i.e., distributions that are invariant under permutation of the indexes.

In this section we will assume that the class $\mathcal{H}$ is binary valued and has 
finite VC-dimension. We write $\VC (\mathcal{H})$ for the VC-dimension of 
$\mathcal{H}$. We show that the stochastic global sequential covering number can be upper 
bounded by $e^{O(\VC (\mathcal{H})\log^2 T)}$ w.h.p. using the \emph{1-inclusion graph} algorithm that was introduced in~\citep{haussler1994predicting}.

Without going into the technical details of the \emph{1-inclusion graph} algorithm, 
we can understand it as a function that maps  $(\mathcal{X}\times \{0,1\})^{t-1}
\times\mathcal{X}\rightarrow \{0,1\}$, for any given $t\ge 1$. For $\mathcal{H}$ 
of finite VC-dimension, for any function 
$\Phi:(\mathcal{X}\times \{0,1\})^{t-1}\times \mathcal{X}\rightarrow \{0,1\}$, 
we define the following quantity (here, we follow the notation 
in~\citep{haussler1994predicting}):
$$
\hat{\hat{M}}_{\Phi,\mathcal{H}}(t)=\sup_{\x^t\in \mathcal{X}^t}\sup_{h\in 
\mathcal{H}}\mathbb{E}_{\sigma}\left[1\{\Phi(\x^{\sigma(t)},h(\{\x^{\sigma(t-1)}\}))
\not=h(x_{\sigma(t)})\}\right],
$$
where $\sigma$ is the uniform random permutation over $[t]$, $\x^{\sigma(t)}=\{\x_{\sigma(1)},\cdots,\x_{\sigma(t)}\}$ and $h(\{\x^{\sigma(t-1)}\})=\{h(\x_{\sigma(1)}),\cdots,h(\x_{\sigma(t-1)})\}$. The main result of~\cite{haussler1994predicting} is stated as follows:

\begin{theorem}[Haussler et al., Theorem 2.3(ii)]
\label{th3}
For any binary valued class $\mathcal{H}$ of finite VC-dimension and for any 
$t\ge 1$, there exists a function $\Phi:(\mathcal{X}\times \{0,1\})^{t-1}\times 
\mathcal{X}\rightarrow \{0,1\}$, i.e., the 1-inclusion graph algorithm, that satisfies
$$\hat{\hat{M}}_{\Phi,\mathcal{H}}(t)\le \frac{\VC (\mathcal{H})}{t}.$$
 \end{theorem}

Our main result for this part is as follows, with the proof presented below Lemma~\ref{lem1}.
 
 \begin{theorem}
\label{th4}
 For any binary valued class $\mathcal{H}$ with finite VC-dimension, 
there exists a global sequential covering set $\mathcal{G}$ of $\mathcal{H}$ w.r.t. 
the class of all $i.i.d.$ distributions over $\mathcal{X}^T$ at scale 
$\alpha=0$ and confidence $\delta$ such that for $T\ge e^9$ we have
$$\log |\mathcal{G}|\le 5\VC(\mathcal{H})\log^2 T+\log T\log(1/\delta)+\log T.$$
 \end{theorem}
 
The main idea of proving Theorem~\ref{th4} is to show that for the 1-inclusion graph predictor $\Phi$ we have w.p. $\ge 1-\delta$ over the sample 
$\x^T\overset{i.i.d}{\sim} \mu$, the cumulative error is upper bounded by $O(\VC (\mathcal{H})\log T+\log(1/\delta))$. Assuming this holds, one will be able to construct 
the covering functions in a similar fashion as in~\cite[Lemma 12]{ben2009agnostic}. 
The bound will follow by counting the error patterns.
However, a direct application of Theorem~\ref{th3} will only give us an \emph{expected} $\VC (\mathcal{H})\log T$ error bound. The main challenge is to establish a high probability error bound with logarithmic dependency on the confidence parameter $1/\delta$. 

Our main proof technique is to exploit the \emph{permutation invariance} of the $1$-inclusion graph predictor, which allows us to relate the  cumulative error to a \emph{reversed} martingale. Using Bernstein's inequality for martingales, we then establish the following key lemma, see Appendix~\ref{app-lem1} for detailed proof.
 \begin{lemma}
\label{lem1}
Let $\Phi:(\mathcal{X}\times\{0,1\})^*\times\mathcal{X}\rightarrow \{0,1\}$ 
and $h:\mathcal{X}\rightarrow \{0,1\}$ be functions such that 
$\Phi$ is permutation invariant on $(\mathcal{X}\times \{0,1\})^*$. 
If for all $t\in [T]$ and $\x^t\in \mathcal{X}^t$ we have:
 $$\mathrm{Pr}_{\sigma_t}\left[\Phi(\x^{\sigma_t(t)}, h(\{\x^{\sigma_t(t-1)}\}))\not=h(\x_{\sigma_t(t)})\right]\le \frac{C}{t},$$
 where $\sigma_t$ is the uniform random permutation on $[t]$ and $C\in \mathbb{N}^+$, then for all $\delta>0$
 and $T\ge e^9$
 $$\mathrm{Pr}_{\sigma_T}\left[\sum_{t=1}^T1\{\Phi(\x^{\sigma_T(t)},h(\{\x^{\sigma_T(t-1)}\}))\not=h(\x_{\sigma_T(t)})\}\ge 4C\log T+\log(1/\delta)\right]\le \delta.$$

 \end{lemma}

\begin{proof}[Proof of Theorem~\ref{th4}]
Let $\Phi$ be the $1$-inclusion graph predictor. We have $\Phi$ is permutation invariant, since the nodes in the $1$-inclusion graph are determined by subsets of $\mathcal{X}$, which does not depend on the order of elements in the set. By symmetries of $i.i.d.$ distributions, for any 
event $A(\x^T)$ on $\x^T\overset{i.i.d.}{\sim} \mu$, 
$\mathrm{Pr}[A(\x^T)]=\mathbb{E}_{\sigma}[\mathrm{Pr}_{\x^T}[A(\x_{\sigma(1)},\cdots,
\x_{\sigma(T)})]]\le \sup_{\x^T}\mathrm{Pr}_{\sigma}[A(\x_{\sigma(1)},\cdots,\x_{\sigma(T)})].$ It is sufficient to show that for any $\x^T\in \mathcal{X}^T$, w.p. $\ge 1-\delta$ over a random permutation $\sigma$ on $[T]$,
$$\sup_{h\in \mathcal{H}}\sum_{t=1}^T1\{\Phi(\x^{\sigma(t)}, h(\{\x^{\sigma(t-1)}\}))\not=h(\x_{\sigma(t)})\}\le 
5\VC (\mathcal{H})\log T+\log(1/\delta).
$$
To see this, we observe that by Sauer's lemma~\citep{shalev2014understanding} there are at most 
$T^{\VC (\mathcal{H})}$ functions of $\mathcal{H}$ restricted on any given $\x^T$. Let now $\delta$ in 
Lemma~\ref{lem1} be $\frac{\delta}{T^{\VC (\mathcal{H})}}$ and
$C=\VC (\mathcal{H})$. When applying Theorem~\ref{th3}  
together with a union bound, the error bound w.p. $\ge 1-\delta$ is of the form
 $5\VC (\mathcal{H})\log T+\log(1/\delta).$
 The upper bound for the size of covering set $\mathcal{G}$ follows  from
below Lemma~\ref{lem2} by taking $\Omega\subset \mathcal{X}^T$ to be the set for 
which $\Phi$ makes at most $5\VC (\mathcal{H})\log T+\log(1/\delta)$ 
accumulated errors, where $\mathrm{Pr}[\Omega]\ge 1-\delta$.
 \end{proof}
 
 \begin{lemma}[From error bound to covering]
\label{lem2}
Let $\mathcal{H}\subset \{0,1\}^{\mathcal{X}}$ be a binary valued class and $\textbf{err}\in \mathbb{N}^+$. For any $\Omega\subset \mathcal{X}^T$, suppose there exists a prediction rule $\Phi$ such that
$\forall h\in \mathcal{H},~\forall \x^T\in \Omega,~\sum_{t=1}^T1\{\Phi(\x^t,h(\{\x^{t-1}\}))
\not=h(\x_t)\}\le \textbf{err}.$
Then, there exists a covering set 
$\mathcal{G}\subset \{0,1\}^{\mathcal{X}^*}$ such that for all $\x^T\in \Omega$ and $h\in \mathcal{H}$  
one can find $g\in \mathcal{G}$ satisfies $g(\x^t)=h(\x_t)$ for all  $t\in [T]$, and
 $$|\mathcal{G}|\le \sum_{t=0}^{\textbf{err}}\binom{T}{t}\le T^{\textbf{err}+1}.$$
 \end{lemma}
 \begin{proof}
For any $I\subset [T]$ with $|I|\le \textbf{err}$, we construct a function 
$g_I$ as follows. Let $\x^t$ be the input, if $t\in I$, we set
 $g_I(\x^t)=1-\Phi(\x^t,g_I(\{\x^{i}\}_{i=1}^{t-1})),$
 else, we set $g_I(\x^t)=\Phi(\x^t,g_I(\{\x^{i}\}_{i=1}^{t-1}))$ where $g_I(\{\x^{i}\}_{i=1}^{t-1})=\{g_I(\x^1),\cdots,g_I(\x^{t-1})\}$. 
We claim that the set $\mathcal{G}$ that consists of all such 
$g_I$s is the desired covering set. To see this, for any 
$h\in \mathcal{H}$ and $\x^T\in \Omega$ we have 
$\sum_{t=1}^T1\{\Phi(\x^t,h(\{\x^{t-1}\}))\not=h(\x_t)\}\le \textbf{err}.$
Let $I$ be the positions $i\in [T]$ for which $\Phi(\x^i,h(\{\x^{i-1}\}))\not=h(\x_i)$, 
where $|I|\le \textbf{err}$. By construction, it is easy to check that 
for all $t\in [T]$ we have $g_I(\x^t)=h(\x_t)$. 
The upper bound for $|\mathcal{G}|$ follows by counting the number of $I$s. See  \cite[Lemma 12]{ben2009agnostic}.
 \end{proof}
 
 The next result follows from Theorems~\ref{th4} and \ref{th1}.
 
 \begin{corollary}
\label{cor1}
 Let $\mathcal{H}$ be a binary valued class with finite VC-dimension, $\mathcal{P}$ be the class of all $i.i.d.$ distributions over $\mathcal{X}^T$ and $T\ge e^9$. If $\ell(\cdot, y)$ is convex, $L$-lipschitz and bounded by $1$ for all $y\in \mathcal{Y}$, then
 $\tilde{r}_T(\mathcal{H},\mathcal{P})\le \sqrt{3T\VC (\mathcal{H})\log^2 T}+O(1).$
 If, in addition, $\ell$ is $\eta$-Mixable then
 $$\tilde{r}_T(\mathcal{H},\mathcal{P})\le \frac{6}{\eta}\VC (\mathcal{H})\log^2 T+O(1).$$
 \end{corollary}

The  first bound of Corollary~\ref{cor1} recovers~\citep{lazaric2009hybrid} 
but with a worse $\log T$ term. This is because the epoch based approach 
of ~\cite{lazaric2009hybrid} applies EWA algorithm on an expert set of size 
$2^{s\VC (\mathcal{H})}$ at each epoch $s$ and the regret is dominated by $\sum_{s=1}^{\log T}O(\sqrt{s2^{s\VC (\mathcal{H})}})=O(\sqrt{T\VC (\mathcal{H})\log T})$. To the best of our  knowledge, the second bound is new. We should remark that our $\log T$ error bound is the key to finding a $O(\log^2 T)$ 
regret while the analysis of~\cite{lazaric2009hybrid} can only give a $O(\sqrt{T})$ bound.

Finally, we bound the regret under the logarithmic loss which follows from
Theorems~\ref{th4} and \ref{th2}.

\begin{corollary}
\label{cor2}
Let $\mathcal{H}$ be a binary valued class with finite VC-dimension, 
$\mathcal{P}$ be the class of all $i.i.d.$ distributions over $\mathcal{X}^T$ 
and $T\ge e^9$. If $\ell$ is Log-loss, then
$\tilde{r}_T(\mathcal{H},\mathcal{P})\le 6\VC (\mathcal{H})\log^2T+O(1)$.
\end{corollary}

To our best knowledge the bound in Corollary~\ref{cor2} is the first known poly-logarithmic bound for $\tilde{r}_T$ under logarithmic loss, while the bound implied by~\citep{bilodeau2021minimax} was only proved for $\bar{r}_T$ and in the weaker \emph{realizable} (and averaged) setting.

\paragraph{Tight bounds for classes with finite Star number.}
In the previous section, we demonstrated that the stochastic global sequential 
covering number of finite VC class is upper bounded w.h.p. by $e^{O(\log^2T)}$. 
We now show that if we assume additional structure on the class,  
we can improve the bound to $e^{O(\log T)}$, matching the naive fixed design 
lower bound for many non-trivial classes (see Section~\ref{sec-lower}).
In Appendix~\ref{app-omega} we show that even for $1$ dimension threshold functions the expected  cumulative error is lower bounded by $\Omega(\log T)$, thus arguing  that the error pattern counting argument as in Lemma~\ref{lem2} cannot provide a bound better than $e^{O(\log^2T)}$. To resolve this issue, we introduce the notion of \emph{Star number} that was used 
originally in~\citep{hanneke2015minimax} for 
analyzing the sample complexity of active learning; 
however, we are using it here in a completely different context. For any binary 
valued class $\mathcal{H}$ and $\x^d\in \mathcal{X}^d$, we say 
$\mathcal{H}$ Star-shatters $\x^d$ if there exist $h,h_1,\cdots,h_d\in 
\mathcal{H}$ such that for all $i,j\in [d]$ with $j\not=i$ we have
$$
h(\x_i)\not=h_i(\x_i)\text{ but }h(\x_j)=h_i(\x_j).
$$
The Star number $\textbf{Star}(\mathcal{H})$ of $\mathcal{H}$ is defined to be the maximum 
number $d$ such that there exists $\x^d$ that is Star-shattered by $\mathcal{H}$. 
Clearly, we have $\textbf{Star}(\mathcal{H})\ge \VC (\mathcal{H})$ for all 
$\mathcal{H}$. 

We now introduce a new notion of \emph{certification}, which is the key for our following arguments. For any sequence $\x^t$ and $h\in \mathcal{H}$, we 
say $\x^{t-1}$ certifies $\x_t$ under $h$ if 
$$\forall f\in \mathcal{H},~\text{if }\forall i\in [t-1],~f(\x_i)=h(\x_i)
\text{ then }f(\x_t)=h(\x_t).$$
 
We have the following property of finite Star number classes w.r.t. certification:
 
\begin{lemma}
\label{lem3}
 If $\mathcal{H}$ has Star number upper bounded by $s$, then for any $\x^t\in \mathcal{X}^t$ and $h\in \mathcal{H}$ we have
 $$\mathrm{Pr}_{\sigma}\left[\{\x_{\sigma(1)},\cdots,\x_{\sigma(t-1)}\}\text{ certifies }\x_{\sigma(t)}\text{ under }h\}\right]\ge 1-\frac{s}{t},$$
 where $\sigma$ is the uniform random permutation over $[t]$.
 \end{lemma}
 \begin{proof}
We only need to show that there are at most $s$ points in $\x^t$ that can 
not be certified by the others under $h$. Suppose otherwise, we have $s+1$ such points. 
Consider the realization of $h$ on these points, by definition of certification, 
we can find the functions $h_1,\cdots,h_{s+1}$ as in the definition of 
Star-shattering. This contradicts the fact that the Star number is upper bounded by $s$.
\end{proof}
 
We now prove a high probability bound on the number of non-certified positions 
for a finite Star number class, which is similar to Lemma~\ref{lem1}.
 
\begin{lemma}
\label{lem4}
Let $\mathcal{H}\subset \{0,1\}^{\mathcal{X}}$ be a class with a finite Star number 
and $T\ge e^9$. Then, with probability $\ge 1-\delta$ over $\x^T$ 
(sampled from some $i.i.d.$ distribution over $\mathcal{X}^T$) for all $h\in \CH$
$$
\sum_{t=1}^T1\{\x^{t-1}\text{ does not certify }\x_t
\text{ under }h\}\le \VC (\mathcal{H})\log T+4\textbf{Star}(\mathcal{H})\log T+\log(1/\delta).
$$
\end{lemma}
\begin{proof}
Note that the event $\{\x^{t-1}\text{ does not certify }\x_t\text{ under }h\}$ can be 
viewed as the event that $\Phi$ makes an error at step $t$ as in Lemma~\ref{lem1} 
(since certification is permutation invariant). By Lemma~\ref{lem3} and Lemma~\ref{lem1} with $C=\textbf{Star}(\mathcal{H})$, we have for all 
$h\in \mathcal{H}$ and $\x^T\in \mathcal{X}^T$ w.p. $\ge 1-\delta$ over uniform
random permutation $\sigma$ on $[T]$
$$\sum_{t=1}^T1\{\x^{\sigma(t-1)}\text{ does not certify }
\x_{\sigma(t)}\text{ under }h\}\le 4\textbf{Star}(\mathcal{H})\log T+\log(1/\delta).
$$
By Sauer's lemma there are at most $T^{\VC (\mathcal{H})}$ functions of $\mathcal{H}$ 
restricted on $\x^T$. Using a union bound and setting $\delta:=\frac{\delta}{T^{\VC (\mathcal{H})}}$ 
in the above expression, we have w.p. $\ge 1-\delta$ over random permutation 
$\sigma$
$$\forall h\in \mathcal{H},~\sum_{t=1}^T1\{\x^{\sigma(t-1)}
\text{ does not certify }\x_{\sigma(t)}\text{ under }h\}\le \VC(\mathcal{H})\log T+
4\textbf{Star}(\mathcal{H})\log T+\log(1/\delta).
$$
The result follows since for any event $A$ over $\x^T$ we have 
(due to symmetries of $i.i.d.$ distribution)
$\mathrm{Pr}_{\x^T}[A(\x^T)]=\mathbb{E}_{\sigma}\mathrm{Pr}[A(\x^{\sigma(T)})]\le 
\sup_{\x^T}\mathrm{Pr}_{\sigma}[A(\x^{\sigma(T)})].$
\end{proof}
 
Lemma~\ref{lem4} allows us to construct the sequential covering set explicitly without relying on error pattern counting as in Lemma~\ref{lem2} as shown next.

\begin{theorem}
\label{th5}
Let $\mathcal{H}$ be a binary valued class with finite Star number. Then, there exists a stochastic sequential covering set $\mathcal{G}$ of $\mathcal{H}$ w.r.t. the class of all $i.i.d.$ distributions over $\mathcal{X}^T$ at scale $\alpha=0$ and confidence $\delta$ such that for $T\ge e^9$ we have:
$$\log |\mathcal{G}|\le 5\textbf{Star}(\mathcal{H})\log T+\log(1/\delta).$$
\end{theorem}

\begin{figure}[h]
    \centering
    \begin{minipage}[t]{1\linewidth}
    \vspace{0pt}
        \begin{tikzpicture}[scale=1, nodes={draw, circle}, ->,
         level 1/.style={sibling distance=7.5em},
        level 2/.style={sibling distance=7.5em},
         level 3/.style={sibling distance=3.9em}]
       \node[fill=gray, label=left:$\frac{\{h_1\cdots h_5\}}{\{g_1\cdots g_8\}}$]{$\textbf{x}_1$}
     child { node[fill=gray, label=left:$\frac{\{h_1\cdots h_3\}}{\{g_1\cdots g_4\}}$] {$\textbf{x}_2$} 
        child { node[fill=gray, label=left:$\frac{\{h_1~h_2\}}{\{g_1~g_2\}}$] {$\textbf{x}_3$} 
            child { node[draw, rectangle]{$\frac{\{h_1\}}{\{g_1\}}$} edge from parent node[draw=none, auto=right] {$0$}}
            child { node[draw, rectangle]{$\frac{\{h_2\}}{\{g_2\}}$} edge from parent node[draw=none, auto=left] {$1$}}
            edge from parent node[draw=none, auto=right] {$0$}
        }
        child { node[label=right:$\frac{\{h_3\}}{\{g_3~ g_4\}}$] {$\textbf{x}_3$} 
            child { node[draw, rectangle]{$\frac{\{h_3\}}{\{g_3~g_4\}}$} edge from parent node[draw=none, auto=right] {$0$}}
            child [missing]
            edge from parent node[draw=none, auto=left] {$1$}
        }edge from parent node[draw=none, auto=right] {$0$}
    }
    child { node[label=right:$\frac{\{h_4~h_5\}}{\{g_5\cdots g_8\}}$] {$\textbf{x}_2$} 
        child [missing]
        child {node[fill=gray, label=right:$\frac{\{h_4~h_5\}}{\{g_5\cdots g_8\}}$] {$\textbf{x}_3$}
            child { node[draw, rectangle]{$\frac{\{h_4\}}{\{g_5~g_6\}}$} edge from parent node[draw=none, auto=right] {$0$}}
            child { node[draw, rectangle]{$\frac{\{h_5\}}{\{g_7~g_8\}}$} edge from parent node[draw=none, auto=left] {$1$}}
            edge from parent node[draw=none, auto=left] {$1$}
        }edge from parent node[draw=none, auto=left] {$1$}
     };
    \end{tikzpicture}
    \end{minipage}

    \begin{minipage}[t]{0\linewidth}
        \vspace{-1.8in}
    \begin{tabular}[t]{|c|c|c|c|c|c|}
     \hline
     & $h_1$ & $h_2$ & $h_3$ & $h_4$ & $h_5$ \\
     \hline
     $\textbf{x}_1$ & $0$ & $0$ & $0$ & $1$ & $1$\\
     \hline
     $\textbf{x}_2$ & $0$ & $0$ & $1$ & $1$ & $1$\\
     \hline
     $\textbf{x}_3$ & $0$ & $1$ & $0$ & $0$ & $1$\\
     \hline
    \end{tabular}
    \end{minipage}
    \hspace{-1.6in}
    \caption{Realization tree of $\mathcal{H}$ defined in left table and partition of $\mathcal{G}$}
    \label{fig:my_label2}
\end{figure}
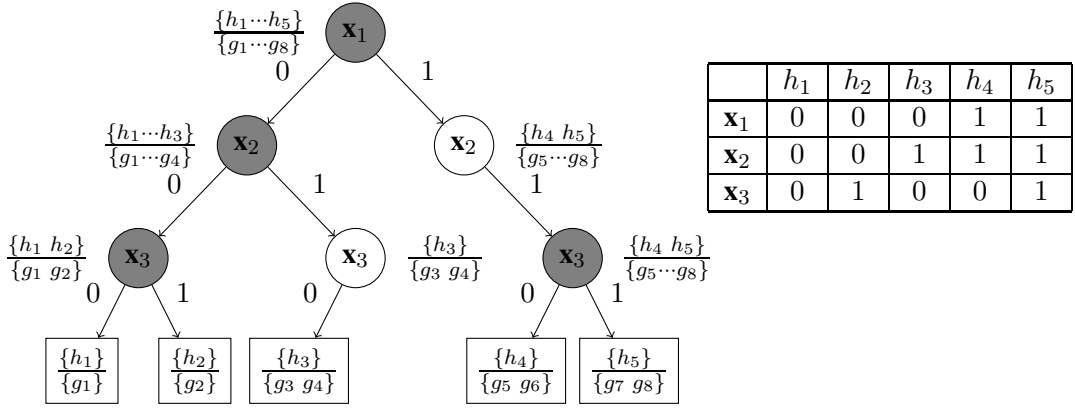

\begin{proof}[Sketch of Proof]
    We only sketch the main idea here (illustrated in Figure~\ref{fig:my_label2}) and refer to Appendix~\ref{app-th5} for full proof. We will construct a set $\mathcal{G}$ of sequential functions with index set $K$. We select $K$ to be \emph{fixed} with $|K|=2^M$, where $M=\lceil 5\textbf{Star}(\mathcal{H})+\log(1/\delta)\rceil$. We will assign sequentially the value $g_k(\x^t)$ for each $k\in K$ after observing the samples $\x^t$. To do so, we maintain a binary tree called the \emph{realization} tree, where at each time after receiving sample $\x_t$ we expand the leaves of the tree with one or two children depending  whether the functions in $\mathcal{H}$ consistent with the leaf agree on $\x_t$ or not. Meanwhile, we also associate with each node in the realization tree a \emph{subset} of $K$ (i.e., a subset of $\mathcal{G}$) in a top-down fashion. Initially, we assign  $K$ to the root. Any time we encounter a node with one child, we pass the set associated with that node to its child, else we split the set into two disjoint subsets of \emph{equal} size and pass them to its two children, respectively. The process is said to have failed, if a node with two children has the size of associated set $\le 1$.  If the procedure does not fail until time $T$, for any $k\in K$, there will be exactly one path (from root to leaf) in the realization tree with nodes that have the associated sets containing $k$. The value of $g_k(\x^t)$ is determined by tracing the path. Clearly, if the procedure does not fail, the set $\{g_k:k\in K\}$ sequentially covers $\mathcal{H}$. The key observation is that if we choose $M$ to be large enough, the procedure does not fail with high probability due to Lemma~\ref{lem4} (since only non-certified positions incur splits on associated sets that reduce the size of the sets for children by $1/2$). 
\end{proof}
\iffalse
We prove Theorem~\ref{th5} in Appendix~\ref{app-th5}, however, here we give a high level description and illustrate the idea of the proof in  Example~\ref{ex3-th5} below. We  use $\CH$ to build a realization tree as shown in
Figure~\ref{fig:my_label2}. This is a binary tree ${\mathcal T}$ with unary and binary nodes of height $O(\log T)$ whp.
By doing it we construct the cover ${\mathcal G}$ of cardinality $K=2^M$ where
$M=\lceil 5\textbf{Star}(\mathcal{H})+\log(1/\delta)\rceil$. For details see Appendix~\ref{app-th5} and the example below.
\fi
 \begin{example}
 \label{ex3-th5}
We illustrate the construction of the realization tree in this example.
We set
$\mathcal{H}=\{h_1,\cdots,h_5\}$, as shown in the table of
Figure~\ref{fig:my_label2} with sample $\x_1,\x_2,\x_3$. The realization tree
is shown in Figure~\ref{fig:my_label2},
where each function $h\in \mathcal{H}$ is consistent with some
path of the tree, and each path has some function $h\in \mathcal{H}$
consistent with it. We assign a subset of $\mathcal{G}$ for each
node in the tree denoted as $\{\cdot\}$. Observe that
if a node has only one child then the
child has the same assigned set as the father, else
we assign an arbitrary partition of the father set with
\emph{equal} sizes to its two children. The final partitions of
the set $\mathcal{G}=\{g_1,\cdots,g_8\}$ are in the leaf nodes
of the tree. In the figure,  binary nodes (i.e., nodes with
two children) are in gray color. The maximum number of binary
nodes in any path is $3$, by selecting $|\mathcal{G}|\ge 2^3=8$,
which guarantees that the assigning procedure does not fail until the leaf.
Each $g_k$ is associated with a unique path from
root to the leaf with (the only) assigned sets on the nodes
that contain $g_k$. The values of $g_k$ are defined to be the
labels of out edges along the path in the obvious way. 
One can verify that $g_1$ covers $h_1$, $g_2$ covers $h_2$, $g_3,g_4$
covers $h_3$, $g_5,g_6$ covers $h_4$, and $g_7,g_8$ covers $h_5$.
Generally, by Lemma~\ref{lem4} the number of binary nodes in any path is of order $O(\log T)$ with high probability (i.e., setting $|\mathcal{G}|=2^{O(\log T)}$ ensures the process success w.h.p.).    
\end{example}

 \begin{corollary}
\label{cor3}
Let $\mathcal{H}$ be a binary valued class with finite Star number, 
$\mathcal{P}$ be the class of all $i.i.d.$ distributions over 
$\mathcal{X}^T$. If $\ell$ is the Log-loss, then 
$\tilde{r}_T(\mathcal{H},\mathcal{P})\le O(\textbf{Star}(\mathcal{H})\log T).$
\end{corollary}

Note that a natural class that has finite star number is the threshold functions $\mathcal{H}=\{1\{x\ge a\}:x,a\in [0,1]\}$, which has Star number $2$. Corollary~\ref{cor3} implies the regret under 
Log-loss is upper bounded by $O(\log T)$. We refer 
to~\citep{hanneke2015minimax} for more non-trivial examples.

We observe that being finite Star number is not a necessary condition to achieve 
a $e^{O(\log T)}$ bound. To see this, consider the class that labels only one 
sample being $1$ and zeros otherwise, which admits a (exact) sequential 
covering of size $T=e^{\log T}$ but has infinite Star number.

\paragraph{Tight bounds with finite Star-Littlestone dimension.}
In this section, we introduce a new complexity measure that we call 
\emph{Star-Littlestone} dimension~\footnote{This is conceptually similar to the VCL tree as 
introduced in the recent paper~\citep{bousquet2021theory}, but they considered a completely different problem using similar idea.}.
The main purpose of this measure is to incorporate the Star number and 
Littlestone dimension that goes beyond  simple finite Star number, and allows us to expand the class of $\CH$ with $e^{O(\log T)}$
cover.

We denote by $\{0,1\}^d_*$ the set of binary 
sequences of length less than or equal to $d$. For any numbers 
$d$ and $s$, we say a binary tree $\tau:\{0,1\}_*^d\rightarrow 
\mathcal{X}$ is Star-Littlestone shattered by $\mathcal{H}$ at star 
scale $s$ if for any path $\epsilon^d\in \{0,1\}_*^d$ 
we have $\textbf{Star}(\mathcal{H}_{\epsilon^d})>s$, where
$\mathcal{H}_{\epsilon^d}=\{h\in \mathcal{H}:\forall t\in [d],~h(\tau(\epsilon^{t-1}))
=\epsilon_t\}$ and $\CH=\cup_{\epsilon^d} \CH_{\epsilon^d}.$
In words, Star-Littlestone shattering implies that the Star number of the class of hypothesis consistent with any path in $\tau$ has Star number greater than $s$. We define the \emph{Star-Littlestone 
dimension} $\SL(s)$ of $\mathcal{H}$ at star scale $s$ to be the maximum number $d$ such that there exists a tree $\tau$ of depth $d$ that is 
Star-Littlestone shattered at star scale $s$ by $\mathcal{H}$. 

Applying Theorem~\ref{th5} and the SOA argument as in~\citep{ben2009agnostic}, we establish our next main theorem with detailed proof in Appendix~\ref{app-th6}.
 
\begin{theorem}
\label{th6}
Let $\mathcal{H}$ be a binary valued class with Star-Littlestone dimension $\SL(s)$ 
at star scale $s$. Then there exists a global sequential 
covering set $\mathcal{G}$ of $\mathcal{H}$ w.r.t. the class of all $i.i.d.$ distributions over $\mathcal{X}^T$ at scale $\alpha=0$ and confidence $\delta$ such that
$$\log |\mathcal{G}|\le O(\max\{\SL(s)+1,s\}\log T+\log(1/\delta)).
$$
\end{theorem}
 
\begin{example}
In this example we show a class $\CH$ that has both infinite Star number and Littlestone dimension but
finite Star-Littlestone dimension.
Let $\mathcal{H}=\{h_{[a,b]}(x)=1\{x\in [a,b]\}:[a,b]\subset [0,1]\}$ 
be the indicators of intervals. It is easy to see that $\mathcal{H}$ 
has both infinite Star number and Littlestone dimension. 
However we can show that it has Star-Littlestone dimension $0$ at star scale 
$4$. To see this, consider any point $x\in [0,1]$ and the 
hypothesis class $\mathcal{H}_1=\{h\in \mathcal{H}:~h(x)=1\}$. 
We show that the Star number of $\mathcal{H}_1$ is $\le 4$. 
For any $5$  points in $[0,1]$ there must be at least 3 points on the 
same side relative to $x$, the restriction of $\mathcal{H}_1$ on such 
points is equivalent to threshold functions (either of form $1\{x\ge a\}$ or $1\{x\le b\}$), thus cannot Star-shatters these $3$ points. This implies that the global sequential covering 
size of $\mathcal{H}$ is upper bounded by $e^{O(\log T)}$ as in Theorem~\ref{th6}.
 \end{example}

 In Appendix~\ref{app-th6}, we also provide a composition theorem (Proposition~\ref{propap1}) for bounding stochastic sequential covering number of classes generated by composition of simple classes. This allows us to obtain $e^{O(\log T)}$ type bounds that go beyond even finite Star-Littlestone dimensional classes.
\begin{remark}
We leave it as an open problem to determine if the upper bound 
$e^{O(\log T)}$ can be achieved for any finite VC-dimensional class. Establishing such a result even for the threshold functions $\mathcal{H}=\{h_{\w}(\x)=1\{\langle \w,\x\rangle\ge a\}:\w,\x\in \mathbb{R}^d,~a\in \mathbb{R}\}$ with $d\ge 2$ seems to be non-trivial.
\end{remark}

\subsection{Real Valued Class with Finite Fat-shattering}
\label{sec-real}

We now assume that $\mathcal{H}\subset [0,1]^{\mathcal{X}}$ is $[0,1]$-valued and 
the fat-shattering number is bounded. 
We discuss this case only briefly here, and the reader is referred to 
Appendix~\ref{app-real} for a detailed discussion.

We first recall the notion of fat-shattering number, which can be viewed as a 
scale sensitive VC-dimension. 
For any class $\mathcal{H}\subset[0,1]^{\mathcal{X}}$, 
we say $\mathcal{H}$ $\alpha$-fat shatters $\x^d\in \mathcal{X}^d$ if there 
exists $s^d\in [0,1]^d$ such that for all $I\subset [d]$ there exists 
$h\in \mathcal{H}$ such that for all $t\in [d]$:
    (1) If $t\in I$, then $h(\x_t)\ge s_t+\alpha$;
    (2) If $t\not\in I$, then $h(\x_t)\le s_t-\alpha$.
Then, the fat shattering number of $\mathcal{H}$ at scale $\alpha$ is defined to 
be the maximum number $d:=d(\alpha)$ such that there exists 
$\x^d\in \mathcal{X}^d$ with $\mathcal{H}$ $\alpha$-fat shatters $\x^d$.

We now state our main result for this part. The proof can be found in Appendix~\ref{app-real}.

\begin{theorem}
\label{th7}
Let $\mathcal{H}$ be a class of functions  $\mathcal{X}\rightarrow [0,1]$ with 
the $\alpha$-fat shattering number $d(\alpha)$.
Then there exists a stochastic global sequential covering set $\mathcal{G}$ of 
$\mathcal{H}$ w.r.t. the class of all $i.i.d.$ distributions 
over $\mathcal{X}^T$ at scale 
$\alpha$ and confidence $\delta$ such that
$$
\log|\mathcal{G}|\le 8d(\alpha/32)(\log T\log(1/\alpha))^4+\log T\log(\log T/\delta)+O(1),
$$
where $O(1)$ hides absolute constant that is independent of 
$\alpha$, $T$ and $\delta$.
\end{theorem}

Note that the main challenge is to obtain a high probability 
error bound of the form $\mathrm{Pr}[\sup_h]$, while a direct application 
of the $\sup_h \mathrm{Pr}[\cdot]$ 
type bound as in~\citep{daniely2011multiclass} does not apply. Our proof in 
Appendix~\ref{app-real} is based on an application 
of the classical symmetric argument and an epoch 
approach similar to~\citep{lazaric2009hybrid}.

We complete this section with two results regarding the 
expected worst case minimax regret.

\begin{corollary}
\label{cor4}
Let $\mathcal{H}$ be a [0,1]-valued class with $\alpha$-fat 
shattering number of order 
$\alpha^{-l}$ for some $l\ge 0$, and $\mathcal{P}$ be a class of all $i.i.d.$ 
distributions over $\mathcal{X}^T$. If $\ell(\cdot, y)$ is convex, 
$L$-lipschitz and bounded by $1$ for all $y\in \mathcal{Y}$, then
$\tilde{r}_T(\mathcal{H},\mathcal{P})\le \tilde{O}((LT)^{(l+1)/(l+2)})$
where $\tilde{O}$ hides a poly-log factor.
\end{corollary}
\begin{proof}
Apply Theorem~\ref{th7} to Theorem~\ref{th1} to find
$\tilde{r}_T(\mathcal{H},\mathcal{P})\le \inf_{0\le\alpha\le 1}
\left\{\alpha LT+\tilde{O}\left(\sqrt{T\alpha^{-l}}\right)\right\}$ 
and taking $\alpha=(LT)^{-1/(l+2)}$  finishes the proof.
\end{proof}

\begin{corollary}
\label{cor5}
Let $\mathcal{H}$ be a [0,1]-valued class with $\alpha$-fat shattering number 
of order $\alpha^{-l}$, and $\mathcal{P}$ be the class of all 
$i.i.d.$ distributions over $\mathcal{X}^T$. If $\ell$ is Log-loss, then
$
\tilde{r}_{T}(\mathcal{H})\le \tilde{O}(T^{l/l+1}).
$
\end{corollary}
\begin{proof}
Applying Theorem~\ref{th7} to Theorem~\ref{th2}, we have
$\tilde{r}_T(\mathcal{H},\mathcal{P})\le \inf_{0\le 
\alpha\le 1}\left\{2\alpha T+\tilde{O}(\alpha^{-1})\right\},
$
and taking $\alpha=T^{-1/(l+1)}$ completes the proof.
\end{proof}

 \section{Lower Bounds For Regret}
\label{sec-lower}

We now provide a general approach for lower bounding the regret 
$\tilde{r}(\CH, \mathcal{P})$ 
using the fixed design regret defined in (\ref{eq-rHx}) and analyzed in
\citep{wu2022precise} as well as \citep{shamir2020logistic, ss21,jss20}.  We first introduce the 
following well known tail bound for the coupon collector problem, 
see e.g.~\cite[Theorem 1.9.2]{doerr2020probabilistic}.

\begin{lemma}
\label{lem7}
Let $X_1,X_2,\cdots$ be $i.i.d.$ samples from the uniform distribution over $[T]$, 
and $\tau$ be the first time such that $[T]\subset X_1^{\tau}$. 
Then for any $c\ge 0$ we have
$\mathrm{Pr}[\tau\ge T\log T+cT]\le e^{-c}.$
\end{lemma}
 
For any function $\Phi$ that maps sequences from 
$\mathcal{X}^*$ to $\mathbb{R}$, 
we say $\Phi$ is monotone if for any $\x^{T}\subset \z^{T_1}$ we have 
$\Phi(\x^T)\le \Phi(\z^{T_1})$, where $\x^{T}\subset \z^{T_1}$ means that for any 
$\bs\in \mathcal{X}$, the number of $\bs$ appearances in $\x^T$ is no more than the number of appearances of $\bs$ in $\z^{T_1}$. We also assume a 
regularity condition for the loss $\ell$ such that for all $\hat{y}_1,\hat{y}_2\in 
\mathcal{Y}$ there exists $y\in \mathcal{Y}$ with 
$\ell(\hat{y}_1,y)\ge \ell(\hat{y}_2,y)$. 
 
\begin{theorem}
\label{th8}
Let $\mathcal{H}$ be any $[0,1]$-valued class. If the fixed design regret 
$r^*_T(\mathcal{H}\mid \x^T)$, as defined in (\ref{eq-rHx}), is monotone over $\x^T$ and 
$\ell$ satisfies the above regularity condition, 
then:
$$
\tilde{r}_T(\mathcal{H},\mathcal{P})\ge (1-O(1/\log T))r^*_{\kappa^{-1}(T)}(\mathcal{H})
\ge (1-O(1/\log T))r^*_{(T/\log T)}(\mathcal{H}),
$$
where $\mathcal{P}$ is the class of all $i.i.d.$ distributions over 
$\mathcal{X}^T$ and $\kappa(T)=T\log T+T\log\log T$.
\end{theorem}
\begin{proof}
Let $\tilde{\x}^T$ be the feature that achieves the maximum of 
$r^*_T(\mathcal{H}\mid \tilde{\x}^T)$ (i.e., $r^*_T(\mathcal{H})$). 
We define the distribution $\nu$ to be the uniform distribution over $\{\tilde{\x}_1,\cdots,\tilde{\x}_T\}$ (with possibly repeated elements). Let $T_1=T\log T+T\log\log T$. We have
\small
\begin{align}
\tilde{r}_{T_1}(\mathcal{H},\mathcal{P})&=\inf_{\phi^{T_1}}\sup_{\mu\in \mathcal{P}}
\mathbb{E}_{\x^{T_1}\sim \mu}\left[\sup_{y^{T_1}}\left(\sum_{t=1}^{T_1}
\ell(\hat{y}_t,y_t)-\inf_{h\in \mathcal{H}}\sum_{t=1}^{T_1}\ell(h(\x_t),y_t)\right)\right]\\
&\ge\inf_{\phi^{T_1}}\mathbb{E}_{\x^{T_1}\sim \nu^{T_1}}\left[\sup_{y^{T_1}}
\left(\sum_{t=1}^{T_1}\ell(\hat{y}_t,y_t)-\inf_{h\in \mathcal{H}}
\sum_{t=1}^{T_1}\ell(h(\x_t),y_t)\right)\right]\\
&\overset{(a)}{\ge} \inf_{\phi^{T_1}}\mathrm{Pr}[\tilde{\x}^T\subset \x^{T_1}]*\mathbb{E}
\left[\sup_{y^{T_1}}\left(\sum_{t=1}^{T_1}\ell(\hat{y}_t,y_t)-\inf_{h\in\mathcal{H}}
\sum_{t=1}^{T_1}\ell(h(\x_t),y_t)\right)\mid \tilde{\x}^T\subset \x^{T_1}\right]\\
&\overset{(b)}{\ge}\mathrm{Pr}[\tilde{\x}^T\subset \x^{T_1}]*\mathbb{E}
\left[\inf_{\phi^{T_1}}\sup_{y^{T_1}}\left(\sum_{t=1}^{T_1}\ell(\hat{y}_t,y_t)-
\inf_{h\in \mathcal{H}}\sum_{t=1}^{T_1}\ell(h(\x_t),y_t)\right)\mid 
\tilde{x}^T\subset \x^{T_1}\right]\\
&=\mathrm{Pr}[\tilde{\x}^T\subset \x^{T_1}]*\mathbb{E}\left[r^*_{T_1}(\mathcal{H}
\mid \x^{T_1})\mid \tilde{\x}^T\subset\x^{T_1}\right]\\
&\overset{(c)}{\ge} \mathrm{Pr}[\tilde{\x}^T\subset \x^{T_1}]r^*_T(\mathcal{H}\mid 
\tilde{\x}^T)\overset{(d)}{\ge} \left(1-\frac{1}{\log T}\right)r^*_T(\mathcal{H}),
\end{align}
\normalsize
where $(a)$ follows by conditioning on the event $\{\tilde{\x}^T\subset\x^{T_1}\}$ 
and observing that the regret is positive for all $\x^{T_1}$, $(b)$ follows 
by $\inf\mathbb{E}\ge \mathbb{E}\inf$, $(c)$ follows from the fact that 
$r^*_{T_1}(\mathcal{H}\mid \x^{T_1})\ge r^*_T(\mathcal{H}\mid \tilde{\x}^T)$ which further follows from the monotonicity of $r^*_T(\mathcal{H}\mid \x^T)$, 
$(d)$ follows by Lemma~\ref{lem7}. 
To complete the proof we set
$T=\kappa^{-1}(T_1)$ and notice that $\kappa^{-1}(T_1)\ge \frac{T_1}{\log T_1}$.
\end{proof}

The following lemma shows the monotonicity for Log-loss~\footnote{Using result in~\cite[Theorem 8.1]{lugosi-book}, one can establish similar result for absolute loss.}.
 
\begin{lemma}
\label{prop2}
 For Log-loss we have
 $r^*_{T_1}(\mathcal{H}\mid \x^{T_1})\ge r^*_T(\mathcal{H}\mid \tilde{\x}^T),$
 so long as $\tilde{\x}^T\subset \x^{T_1}$.
 \end{lemma}
 \begin{proof}
 Note that for any $\x^T$, we have~\citep{jss20}
 $$
r^*_T(\mathcal{H}\mid \x^T)=\log\sum_{y^T}\sup_{h\in \mathcal{H}}\prod_{t=1}^Th(\x_t)^{y_1}
(1-h(\x_t))^{1-y_t}.$$
Therefore, any permutation over $\x^T$ does not change the value $r^*_T$. 
Now, suppose $\tilde{\x}^T\subset \x^{T_1}$; we can permute $\x^{T_1}$ so 
that the first $T$ samples matches with $\tilde{\x}^T$. The result follows 
from the fact that playing more rounds does not decrease the regret 
(if so we can replace the strategy for short rounds with the strategy 
for long rounds which reduces regrets thus a contradiction).
 \end{proof}
 
Finally, we apply the above general lower bound to the expected 
worst case minimax regret. 

\begin{corollary}
\label{cor6}
Assume $\ell$ is the Log-loss. If $r^*_T(\mathcal{H})\ge C\log^{\alpha} T$ then 
$$
\tilde{r}_T(\mathcal{H},\mathcal{P})\ge C\log^{\alpha} T-o(\log^{\alpha} T),
$$
where $\mathcal{P}$ is the class of $i.i.d.$ distributions. If $r^*_T(\mathcal{H})\ge CT^{\alpha}$, then 
$$
\tilde{r}_T(\mathcal{H},\mathcal{P})\ge \frac{CT^{\alpha}}{\log^{\alpha} T}
-o(T^{\alpha}/\log^{\alpha} T).
$$
\end{corollary}
 We refer to~\citep{wu2022precise} for the Lower bounds on $r^*_T(\mathcal{H})$ of various classes $\mathcal{H}$ under Log-loss.

Note that a lower bound on $\tilde{r}_T(\mathcal{H},\mathcal{P})$ does not 
necessarily imply a lower bound for $\bar{r}_T(\mathcal{H},\mathcal{P})$. 
However, in Appendix~\ref{app-th9} we 
prove the following lower bound for $\bar{r}_T$ under Log-loss.

\begin{theorem}
\label{th9}
There exists a class $\mathcal{H}$ of finite VC-dimension, such that if 
$\ell$ is Log-loss and $\mathcal{P}$ is the class of 
$i.i.d.$ distributions over $\mathcal{X}^T$, then
$$
\bar{r}_T(\mathcal{H},\mathcal{P})\ge (1-o(1))\VC(\mathcal{H})
\log(T/\VC (\mathcal{H})).
$$
\end{theorem}
 
Our proof, presented in Appendix~\ref{app-th9}, actually establishes the lower 
bound in the \emph{realizable} case, which is stronger than the $\VC +\log T$ lower bounds as 
shown in~\citep{bhatt2021sequential}.

\section{Conclusions and Further Results}

In this paper we introduced a general minimax regret called 
the expected worst case minimax regret when the
features are generated by a stochastic source. This new minimax regret recovers previously known online minimax regrets in a unified way. We analyzed the regret via a novel concept of stochastic global sequential covering and provide tight bounds on the covering size.

Due to limited space some extensions of our results are presented only in
the Appendix.  In particular, we discuss more interesting examples 
(see Appendix~\ref{app-th6}),
distributions with sufficient symmetry (see Appendix~\ref{app-symmetry}), 
and extensions to the distribution non-blind case (see Appendix~\ref{app-nonblind}).

\bibliography{main.bib}

%\newpage

\appendix

\section{Proof of Lemma~\ref{lem1}}
\label{app-lem1}

For any $t\in [T]$, we denote by $I_t=1\{\Phi(\x^{\sigma(t)},h(\{\x^{\sigma(t-1)}\}))
\not=h(\x_{\sigma(t)})\}$
where $\sigma$ is a uniform random permutation over $[T]$. For all $t\in [T]$,
we define the \emph{reversed} sequence of indicators $I_t'=I_{T-t+1}$.
We observe that for all $t\in [T]$,
the indicator $I_t'$ only depends on the realizations 
of $\x_{\sigma(T)},\x_{\sigma(T-1)},
\cdots,\x_{\sigma(T-t+1)}$ since $\Phi$ is permutation invariant on
$\x_{\sigma(1)},\cdots,\x_{\sigma(T-t)}$. Therefore,
$$\mathbb{E}[I_t'\mid I_1',\cdots,I_{t-1}']=\mathbb{E}[I_t'\mid \x_{\sigma(T)},
\cdots,\x_{\sigma(T-t+2)}]\le \min\left\{\frac{C}{T-t+1}, 1\right\},$$
where the last inequality follows from the assumption of $\Phi$ 
and noticing that conditioning
on $\x^{\sigma(T)}_{\sigma(T-t+2)}$ the permutation $\sigma$ restricted on
$\x^T\backslash \{\x_{\sigma(T)},\cdots,\x_{\sigma(T-t+2)}\}$ is also a uniform random
permutation. For any realization $I'_1,\cdots,I_{t-1}'$, we define 
$I''_t=I'_t-\mathbb{E}[I_t'\mid I'_1,\cdots,I_{t-1}']$. 
We now observe that the indicators $I_t''$ are (Doob) martingale differences, i.e., we have $\forall t\in [T]$:
$$
\mathbb{E}[I''_t\mid I''_1,\cdots,I''_{t-1}]=0.$$
By the Bernstein inequality for martingales~\cite[Lemma A.8]{lugosi-book}, we find
$$\mathrm{Pr}\left[\sum_{t=1}^TI_t''>k\text{ and }
\Sigma^2\le v\right]\le e^{-\frac{k^2}{2(v+k/3)}},
$$
where
$$
\Sigma^2=\sum_{t=1}^T\mathbb{E}[{I''_{t}}^{2}\mid I''_{1},\cdots,I''_{t-1}].
$$
We now observe that conditioning on $I_1',\cdots,I_{t-1}'$ the indicator $I_t'$ is a
Bernoulli random variable with parameter $p_t\le \min\left\{\frac{C}{T-t+1}, 1\right\}$. This implies that if $I_t'=1$ then $I''_t\le 1$ and if $I_t'=0$ then $|I''_t|\le p_t$.
Using elementary algebra, there exists an absolute constant 
$c_1< 2$ such that with probability
$1$ we have
$$
\sum_{t=1}^T\mathbb{E}[{I''_t}^2\mid I_1',
\cdots,I_{t-1}']\le \sum_{t=1}^Tp_t+(1-p_t)p_t^2\le C\log T+2C+c_1.
$$
Plugging it into the Bernstein inequality  with 
$k=2(C\log T+2C+c_1)+\log(1/\delta))$ and $v=C\log T+2C+c_1$, with probability
$\ge 1-\delta$, we have for some constant $c_2< 5$,

\begin{align*}
    \sum_{t=1}^TI_t=\sum_{t=1}^TI'_t&\le \sum_{t=1}^TI''_t+\mathbb{E}[I_t'\mid
I'_1,\cdots,I_{t-1}']\\
&\le k+\sum_{t=1}^T\min\left\{\frac{C}{T-t+1},
1\right\}\le 3C\log T+4C+\log(1/\delta)+c_2.
\end{align*}

Here, we used the following elementary inequality:
$$\forall a,b\ge 0,~\frac{(2a+b)^2}{2(a+(2a+b)/3)}\ge b.$$
The Lemma now follows from the fact that $C\log T\ge 4C+5$ when $T\ge e^9$ and $C\ge 1$.

\section{Real Valued Class with Finite Fat-shattering}
\label{app-real}
We first introduce the notion of local $\alpha$-covering. We say that a class $\mathcal{F}$ locally $\alpha$-covers $\mathcal{H}$ at $\x^T\in 
\mathcal{X}^T$ if for all $h\in \mathcal{H}$ there exists 
$f\in \mathcal{F}$ such that:
$$\forall t\in [T],~|h(\x_t)-f(\x_t)|\le \alpha.$$
Here, we also assume that $\mathcal{F}\subset \mathcal{H}$ 
(we can always convert $\alpha$-covering set $\mathcal{F}$ of $\mathcal{H}$ to 
a $2\alpha$-covering set $\tilde{\mathcal{F}}\subset \mathcal{H}$
such that $|\tilde{\mathcal{F}}|\le |\mathcal{F}|$).

The following lemma upper bounds the local $\alpha$-covering size 
w.r.t. the $\alpha$-fat shattering
number of $\mathcal{H}$, which is due to~\cite{alon1997scale}.

\begin{lemma}
\label{lem5}
Suppose the $\alpha$ fat-shattering number of $\mathcal{H}$ is 
$d(\alpha)$. Then for all $\x^T\in \mathcal{X}^T$ there exists $\mathcal{F}$ 
(which depends on $\x^T$) that locally $\alpha$-covers $\mathcal{H}$ at $\x^T$ such that:
$$|\mathcal{F}|\le 2\left(T\left(\frac{2}{\alpha}+1\right)^2\right)^{\lceil 
d(\alpha/4)\log\left(\frac{2eT}{\alpha d(\alpha/4)}\right)\rceil}\le 
2^{d(\alpha/4)(\log^2 T+2\log^2(1/\alpha)+O(1))}.$$
\end{lemma}

We now prove Theorem~\ref{th7}, which we repeat  below.

{\bf Theorem~\ref{th7}}. {\it
Let $\mathcal{H}$ be a class of functions $\mathcal{X}\rightarrow [0,1]$ and  
let $\alpha$-fat shattering number of $\mathcal{H}$ be $d(\alpha)$. 
Then, there exists a stochastic sequential covering set 
$\mathcal{G}$ of $\mathcal{H}$ w.r.t. 
the class of all $i.i.d.$ distributions over $\mathcal{X}^T$ at scale $\alpha$ and 
confidence $\delta$ such that: 
$$
\log|\mathcal{G}|\le 8d(\alpha/32)(\log T\log(1/\alpha))^4+
\log T\log(\log T/\delta)+O(1),$$
where $O(1)$ hides absolute constant that is independent of $\alpha$, $T$, and $\delta$.
}

Our proof of Theorem~\ref{th7} is based on the following key lemma (which is an application of the classical symmetric argument), and an epoch approach similar to~\cite{lazaric2009hybrid}.

\begin{lemma}
\label{lem6}
Let $\mathcal{H}\subset [0,1]^{\mathcal{X}}$ be a class with $\alpha$-fat 
shattering number $d(\alpha)$. Let $S_1,S_2$ be two $i.i.d.$ samples from 
the same distribution over $\mathcal{X}$, both of size $k$. For any $S_i$ with 
$i\in \{1,2\}$, we define a distance for all $h_1,h_2\in\mathcal {H}$ as:
$$
d_{S_i}^{\alpha}(h_1,h_2)=\sum_{s\in S_i}1\{|h_1(s)-h_2(s)|\ge \alpha\}.$$
Then
$$
\mathrm{Pr}_{S_1,S_2}\left[\exists h_1,h_2\in 
\mathcal{H}~s.t.~d_{S_1}^{\alpha}(h_1,h_2)
=0\text{ and }d_{S_2}^{4\alpha}(h_1,h_2)\ge r\right]
\le 2^{\tilde{O}(d(\alpha/8))-r},
$$
where $\tilde{O}(d(\alpha/8))=2d(\alpha/8)(\log^2 k+2\log^2(1/\alpha)+O(1))$.
\end{lemma}
\begin{proof}
We use a symmetric argument. We denote by $A$ the event that 
$\exists h_1,h_2\in \mathcal{H}$ such that $d_{S_1}^{\alpha}(h_1,h_2)=0$ but 
$d_{S_2}^{4\alpha}(h_1,h_2)\ge r$. Let $\sigma$ be a random permutation that 
switches the $i$th positions of $S_1,S_2$ w.p. $\frac{1}{2}$ and independently for different $i\in[k]$. By symmetries, it
is sufficient to fix $S_1, S_2$ and upper bound
$\mathrm{Pr}_{\sigma}[A[\sigma(S_1,S_2)]].$
By Lemma~\ref{lem5}, we know that there exists a set $\mathcal{F}$ 
that $\alpha/2$-covers $\mathcal{H}$ on $S_1\cup S_2$ with:
$$|\mathcal{F}|\le 2^{d(\alpha/8)(\log^2k+2\log^2(1/\alpha)+O(1))}.$$
If the event $A$ happens, then there exist $f_1,f_2\in \mathcal{F}$ 
such that (using property of covering):
$$d_{S_1}^{2\alpha}(f_1,f_2)=0\text{ but }d_{S_2}^{3\alpha}(f_1,f_2)\ge r.$$
Clearly, in order for $A$ to happen, any position $s\in S_2$ such that 
$|f_1(s)-f_2(s)|\ge 3\alpha$ must not be switched to $S_1$ under $\sigma$, 
which happens with probability upper bounded by $2^{-r}$. 
Applying union bound over all pairs of $\mathcal{F}$, we have
$$\mathrm{Pr}_{S_1,S_2}[A]\le 2^{2d(\alpha/8)
(\log^2 k+2\log^2(1/\alpha)+O(1))-r}
$$
which completes the proof.
\end{proof}

\begin{proof}[Proof of Theorem~\ref{th7}]
We partition the time horizon into epochs, where each epoch $s$ ranges from time step $2^{s-1},\cdots,2^s-1$. For each epoch $s$, we will construct a covering set $\mathcal{G}_s$. The global covering set $\mathcal{G}$ will be constructed by considering all the combinations of functions in $\mathcal{G}_s$ with $s\in \{1,\cdots,\lceil\log T\rceil\}$.

For any epoch $s$, we construct $\mathcal{G}_s$ as follows. 
Let $\mathcal{F}\subset \mathcal{H}$ be the local $\alpha$-covering set on the 
samples $\x_1,\cdots,\x_{2^{s-1}-1}$. By Lemma~\ref{lem5}, we have
$$|\mathcal{F}|\le 2^{d(\alpha/4)(s^2+2\log^2(1/\alpha))+O(1)}.$$ 
Let 
$$r_s=2d(\alpha/8)(s^2+2\log^2(1/\alpha)+O(1))+\log(\log T/\delta). 
$$
By Lemma~\ref{lem6} w.p. $\ge 1-\frac{\delta}{\log T}$ for any 
$h\in \mathcal{H}$ there exists $f\in \mathcal{F}$ such that 
$f$ $4\alpha$-covers $h$ on
samples $\x_{2^{s-1}},\cdots,\x_{2^s-1}$ except $r_s$ positions 
(the $f\in \mathcal{F}$ that $\alpha$-covers $h$ on $\x^{2^{s-1}-1}$ is the 
desired function since $\mathcal{F}$ is a local $\alpha$-covering). 
Let $K$ be a discretization of interval $[0,1]$ such that for any $a\in [0,1]$ there exists $k\in K$ so that $|a-k|\le 4\alpha$. We have $|K|\le \lceil\frac{1}{8\alpha}\rceil$. Now, for any 
$I\subset \{2^{s-1},\cdots,2^{s}-1\}$ with $|I|\le r_s$,
$\{k_i\}_{i\in I}\in K^{|I|}$ and $f\in \mathcal{F}$, 
we construct a function $f_{I, k^{|I|}}$ as follows:
\begin{itemize}
    \item[1.] If $t\in I$, we set $f_{I, k^{|I|}}(\x_t)=k_t$;
    \item[2.] If $t\not\in I$, we set $f_{I,k^{|I|}}(\x_t)=f(\x_t)$.
\end{itemize}
The class $\mathcal{G}_s$ is defined as the class of all 
such $f_{I,k^{|I|}}$. By definition of $r_s$ and by Lemma~\ref{lem6}, we have 
w.p. $\ge 1-\frac{\delta}{\log T}$, for all $h\in \mathcal{H}$ 
there exists $g\in \mathcal{G}_s$ such that for all $t\in \{2^{s-1},\cdots,2^s-1\}$
we have:
$$|g(\x_t)-h(\x_t)|\le 4\alpha.$$
We now observe that:
$$|\mathcal{G}_s|\le |\mathcal{F}|\cdot (2^s|K|)^{r_s+1}\le 
2^{3d(\alpha/8)((s\log(1/\alpha))^3)+
\log(\log T/\delta)+O(1)}.$$

We now construct the global covering set $\mathcal{G}$ as follows. For any 
index $(j_1,\cdots,j_{\lceil\log T\rceil})$ with $j_s\in [|\mathcal{G}_s|]$, we define a 
function $g$ such that it uses the $j_s$ function in $\mathcal{G}_s$ to make 
prediction during epoch $s$. By union bound on the epochs, we have w.p. 
$\ge 1-\delta$ for any $h\in \mathcal{H}$, 
there exists $g$ such that:
$$\forall t\in [T],~|h(\x_t)-g(\x^t)|\le 4\alpha.$$
This implies that $\mathcal{G}$ is a $4\alpha$ global sequential covering set of 
$\mathcal{H}$. Thus
$$|\mathcal{G}|=\prod_{s=1}^{\lceil\log T\rceil}|\mathcal{G}_s|
\le 2^{4d(\alpha/8)(\log T\log(1/\alpha))^4+\log T \log(\log T/\delta)+O(1)}.
$$
The result follows by taking $\alpha$ in the above expression to be $\alpha/4$.
\end{proof}

\section{Proof of Theorem~\ref{th5}}
\label{app-th5}
We will construct a covering set $\mathcal{G}$ directly without relying on 
the error pattern counting as in Lemma~\ref{lem2}. This is the key to removing the extra $\log T$ factor. We will introduce a set $K$ to index the functions in $\mathcal{G}$, we assume that $K$ is fixed and $|K|=2^M$ for some $M$ to be chosen later. For any $k\in K$, we will construct a \emph{sequential function} $g_k$ as follows. 

Let $\textbf{x}^T$ be a realization of the sample from an $i.i.d.$ source. The realization tree $\mathcal{T}$ of $\mathcal{H}$ on $\textbf{x}^T$ is a leveled binary tree of depth $T+1$, with each node at level $t$ being labeled  $\x_t$ (where level $1$ has only the root $v_1$), each left edge being labeled $0$ and each right edge being labeled $1$, such that any node $v_t\in \mathcal{T}$ at level $t$ has left (respectively right) child if and only if there exist $h\in \mathcal{H}$ such that $h(\x_t)=0$ (respectively $h(\x_t)=1$) and $h(\x_i)=L(v_i\rightarrow v_{i+1})$ for all $i\le t-1$, where $v_1\rightarrow v_2\rightarrow \cdots \rightarrow v_t=v$ is the path from root $v_1$ to $v$ and $L$ is the edge label function. Note that different realizations of $\x^T$ will result in different realization trees.

We now assign values of the functions $g_k$ with $k\in K$ using the following procedure. For any node $v$ in the realization tree $\mathcal{T}$, we will associate a set $\mathcal{K}(v)\subset K$ using the following rule (starting from root):
\begin{itemize}
    \item[1.] If $v$ is the root, then $\mathcal{K}(v)=K$;
    \item[2.] If $v$ has only one child $u$, then $\mathcal{K}(u)=\mathcal{K}(v)$;
    \item[3.] If $v$ has two children $u_1,u_2$, we assign the sets to $u_1,u_2$ being an arbitrary partition of $\mathcal{K}(v)$ of \emph{equal} sizes, i.e., $|\mathcal{K}(u_1)|=|\mathcal{K}(u_2)|$, $\mathcal{K}(u_1)\cap \mathcal{K}(u_2)=\emptyset$ and $\mathcal{K}(u_1)\cup \mathcal{K}(u_2)=\mathcal{K}(v)$.
\end{itemize}
Clearly, the value $\mathcal{K}(v)$ for any node $v$ at level $t$ can be determined with only the realization of $\x^t$ and the values of $\mathcal{K}$ of all nodes at level $t$ form a partition of $K$. The procedure $\mathcal{K}$ fails if there exists some node $v$ with two children such that $|\mathcal{K}(v)|<2$. Suppose the procedure $\mathcal{K}$ does not fail. We have for any $k\in K$, there exists a unique path $v_1\rightarrow v_2\rightarrow\cdots\rightarrow v_{T+1}$ with $v_1$ being the root, such that for all $t\le T+1$ we have $k\in\mathcal{K}(v_t)$. For any such $k$, we assign the value of $g_k$ on $\x^t$ as:
$$g_k(\x^t)=L(v_t\rightarrow v_{t+1}),$$
where $L$ is the edge label function as discussed above. If the procedure $\mathcal{K}$ fails at some node $v_t$, we assign the value of $g_k(\x^j)$ arbitrarily for $j\ge t$.

By definition of the realization tree, for any $h\in \mathcal{H}$ there must be a unique path $v_1\rightarrow\cdots\rightarrow v_{T+1}$, with $v_1$ being root such that $h(\x_t)=L(v_t\rightarrow v_{t+1})$ for all $t$. Therefore, if the procedure $\mathcal{K}$ does not fail, then for $k\in \mathcal{K}(v_{T+1})$, we have $h(\x_t)=g_k(\x^t)$ for all $t\le T$ by definition of $g_k$. We now show that by setting $M=\lceil 5\textbf{Star}(\mathcal{H})+\log(1/\delta)\rceil$, w.p. $\ge 1-\delta$ over $\x^T$, the procedure $\mathcal{K}$ will not fail, thus proving that the class $\mathcal{G}=\{g_k:k\in K\}$ is a stochastic sequential covering of $\mathcal{H}$ with confidence $\delta$. To see this, we note that the procedure $\mathcal{K}$ fails at node $v_t$ at level $t$ if and only if there are $\ge M+1$ nodes with two children in the (unique) path $v_1\rightarrow \cdots\rightarrow v_t$, where $v_1$ is root, since only rule $3$ will reduce the size of value of $\mathcal{K}$ by $1/2$. Assume now the procedure $\mathcal{K}$ fails at node $v_t$. Let $h\in \mathcal{H}$ be a function such that $h(\x_i)=L(v_i\rightarrow v_{i+1})$ for all $i\le t$, which must exist by definition of realization tree. Since any node $v_j$ in the path $v_1\rightarrow\cdots\rightarrow v_t$ with two children implies $\x^{j-1}$ \emph{does not} certify $\x_j$ under $h$, we have that there are at least $M+1$ positions $j$ (with $j\le t$) such that $\x^{j-1}$ does not certify $\x_j$ under $h$. By Lemma~\ref{lem4} and selection of $M$, this happens with probability $\le \delta$. This completes the proof.

\section{Proof of Theorem~\ref{th6} and more examples}
\label{app-th6}
\begin{proof}[Proof of Theorem~\ref{th6}]
The proof will incorporate the SOA argument as in~\citep{ben2009agnostic} 
and the result from Theorem~\ref{th5}. For notational convenience, we denote $d=\mathsf{SL}(s)+1$. For any $I\subset [T]$ with $|I|\le d$, 
we will construct a set $\mathcal{G}_{I}$. Let $\Phi$ be the SOA predictor (similar to~\cite[Algorithm 1]{ben2009agnostic}) that predicts the label for which the remaining consistent subclass has maximum Star-Littlestone dimension at star scale $s$, if both subclasses have $\mathsf{SL}$ dimension $0$ we predict the label for which the remaining consistent subclass has maximum Star number (and break ties arbitrarily). We now construct functions in $\mathcal{G}_I$ as follows. The predictions of functions in $\mathcal{G}_I$ are partitioned into 2 phases (start with phase 1). At phase 1, all the functions in $\mathcal{G}_I$ use the same prediction rule as in Lemma~\ref{lem2}, that is, if we are at time step $t\in I$, we predict using $1-\Phi$, else we use $\Phi$ to predict, where $\Phi$ is the SOA prediction rule described above. We enter phase 2 if the remaining consistent class has Star number upper bounded by $s$; we 
then construct the prediction functions in $\mathcal{G}_I$ as in Theorem~\ref{th5} 
with $\textbf{Star}(\mathcal{H})=s$, confidence 
$\delta/T^{d+1}$ and $|\mathcal{G}_I|\le e^{5s\log T+\log(T^{d+1}/\delta)}$. The covering class 
$\mathcal{G}$ is defined to be:
$$\mathcal{G}=\bigcup_{I\subset [T],~|I|\le d}\mathcal{G}_I.$$
By Theorem~\ref{th5} with $\textbf{Star}(\mathcal{H})=s$ and 
$\delta=\delta/T^{d+1}$ and computing the number of $I$s, we have
$$
|\mathcal{G}|\le T^{d+1}e^{5s\log T+\log(T^{d+1}/\delta)}\le 
e^{O(\max\{d,s\}\log T+\log(1/\delta))}.
$$

We now show that $\mathcal{G}$ is indeed a stochastic sequential 
covering of $\mathcal{H}$ with confidence $\delta$. Let $\mathcal{H}_I$ be the (\emph{random}) subclass of functions in $\mathcal{H}$ that are consistent with $\Phi$ with error pattern $I$ before entering phase $2$~\footnote{Here, phase $1$ and $2$ corresponds to that the functions in $\mathcal{H}$ consistent with $h$ on current sample has Star number $> s$ and $\le s$, respectively.} (it is possible that $h$ remains on phase $1$ until time $T$). Note that all functions in $\mathcal{H}_I$ agree on samples at phase 1. 
Note also that, with probability $1$ we have 
$\mathcal{H}=\bigcup_{I\subset [T],|I|\le d}\mathcal{H}_I$. To see this, we note that if $h$ disagreed with the SOA then the remaining consistent class has $\mathsf{SL}(s)$ decreased by at least $1$ (similar to the argument as in~\cite[Lemma 10]{ben2009agnostic}) or has Star number $\le s$ if the current consistent class has $\mathsf{SL}(s)=0$. This implies that any $h\in \mathcal{H}$ can be disagreed with SOA at most $d$ times before entering phase $2$, which must be in some $\mathcal{H}_I$ with $|I|\le d$. Now, for any $I$ with $|I|\le d$ we need to show that:
$$\mathrm{Pr}[\mathcal{G}_I\text{ covers }\mathcal{H}_I]
\ge 1-\frac{\delta}{T^{d+1}}.
$$
Note that the main difficulty here is that $\mathcal{H}_I$ 
is a \emph{random} subset. 
We show that conditioning on any realization of $\mathcal{H}_I$, 
the above inequality holds (the inequality will then hold by 
law of total probability). 
This follows from Theorem~\ref{th5} by noticing that the samples in phase 2 
are still $i.i.d.$ and independent of samples in phase $1$, and $\mathcal{G}_I$ 
trivially covers $\mathcal{H}_I$ in phase $1$ by definition of 
$\mathcal{G}_I$ and $\mathcal{H}_I$. The theorem will now follow by a union bound on all the $I$s.
\end{proof}

We provide the following useful proposition that allows us to prove $e^{O(\log T)}$ type bounds for function classes that are generated by composition of simple classes with $e^{O(\log T)}$ covering.

\begin{proposition}
\label{propap1}
    Let $\mathcal{H},\mathcal{H}_1,\cdots, \mathcal{H}_m$
be binary valued classes over the same domain $\mathcal{X}$ such
that there exists function $f:\{0,1\}^m\rightarrow\{0,1\}$ so that for all
$h\in \mathcal{H}$ there exist $h_i\in \mathcal{H}_i$ such that
$\forall \x\in \mathcal{X},~h(\x)=f(h_1(\x),\cdots,h_2(\x))$.
If $\forall i\in [m],~\mathcal{H}_i$ admit a statistical sequential covering set
$\mathcal{G}_i$ at confidence $\delta/m$ w.r.t. the same class $\mathcal{P}$, then $\mathcal{H}$ admits a
statistical sequential covering set $\mathcal{G}$ w.r.t. $\mathcal{P}$ at confidence $\delta$ such that
$$|\mathcal{G}|\le \prod_{i=1}^m|\mathcal{G}_i|.$$
\end{proposition}
\begin{proof}
    For any tuple of indexes $(j_1,\cdots,j_m)$ with $j_i\in [|\mathcal{G}_i|]$, we construct a function $g$ such that:
    $$g(\x^t)=f(g_{j_1}(\x^t),\cdots,g_{j_m}(\x^t)),$$
    where $g_{j_i}$ is the $j_i$th function in $\mathcal{G}_i$. The covering set $\mathcal{G}$ is defined to be the class containing of all such functions $g$. For any function $h\in \mathcal{H}$, there exist $h_1,\cdots,h_m$ such that for all $\x\in \mathcal{X}$, $h(\x)=f(h_1(\x),\cdots h_m(\x))$. By union bound and definition of stochastic sequential covering of $\mathcal{G}_i$, w.p. $\ge \delta$ over $\x^T$, for all $i\in [m]$ there exist $g_{j_i}\in \mathcal{G}_i$ such that $\forall t\in [T],~g_{j_i}(\x^t)=h_i(\x_t)$. One can verify that the function $g$ corresponding to $(j_1,\cdots,j_m)$ is the desired function covers $h$ on $\x^T$.
\end{proof}
\begin{example}
    Let 
$$\mathcal{H}=\{h_B(\x)=1\{\x\in B\}:B=\prod_{i=1}^d[a_i,b_i]\subset 
\mathbb{R}^d\}
$$
be the class of indicators of rectangular cuboids in $\mathbb{R}^d$.
Note that $\mathcal{H}$ has infinite Star-Littlestone dimension for any finite star scale when $d\ge 2$. However, $\mathcal{H}$ can be expressed as a function in terms of indicators of
intervals. Applying the Proposition~\ref{propap1} and Theorem~\ref{th6}
we obtain a covering set $\mathcal{G}$ of $\mathcal{H}$ with $\log |\mathcal{G}|\le
O(d\log T+d\log(d/\delta))$. This implies a regret bound of mixable losses
(including logarithmic loss) of order $O(d\log T+d\log d)$.
\end{example}

\section{Proof of Theorem~\ref{th9}}
\label{app-th9}

We choose $\mathcal{H}$ to be the class supported over $[T/\log^2 T]\overset{\text{def}}{=}\{1,2,\cdots,\lceil T/\log^2T \rceil\}$ with 
functions that take value $1$ on at most $d$ positions and zeros otherwise. 
It is easy to see that this class has VC-dimension $d$ and 
$|\mathcal{H}|\ge \binom{\lceil T/\log^2 T \rceil}{d}\ge \left(\frac{T}{d\log^2 T}\right)^d$. 
Let $\nu$ be uniform distribution over $[T/\log^2 T]$. We observe that:
\begin{align*}
\bar{r}_T(\mathcal{H},\mathcal{P})&=\inf_{\phi^T}\sup_{\mu\in \mathcal{P}}
\sup_{h}\mathbb{E}_{\x^T\sim \mu}\left[\sup_{y^T}\sum_{t=1}^T\ell(\hat{y}_t,y_t)-
\ell(h(\x_t),y_t)\right]\\
&\ge \inf_{\phi^T}\sup_{h\in \mathcal{H}}\mathbb{E}_{\x^T\sim \nu^T}
\left[\sup_{y^T}\sum_{t=1}^T\ell(\hat{y}_t,y_t)-\ell(h(\x_t),y_t)\right]\\
&\overset{(a)}{\ge}\inf_{\phi^T}\sup_{h\in \mathcal{H}}\mathbb{E}_{\x^T\sim \nu^T}
\left[\sum_{t=1}^T\ell(\hat{y}_t,h(\x_t))\right]-O(1)\\
&\ge \mathrm{Pr}[[T/\log^2 T]\subset \x^T]\inf_{\phi^T}\sup_{h\in \mathcal{H}}
\mathbb{E}_{\x^T\sim \nu^T}\left[\sum_{t=1}^T\ell(\hat{y}_t,h(\x_t))
\mid [T/\log^2 T] \subset \x^T\right]-O(1)\\
&\overset{(b)}{\ge} (1-o(1))\inf_{\phi^T}\mathbb{E}_{h\sim \mathcal{H}}
\mathbb{E}_{\x^T\sim \nu^T}\left[\sum_{t=1}^T\ell(\hat{y}_t,h(\x_t))
\mid [T/\log^2 T] \subset \x^T\right]-O(1)\\
&\overset{(c)}{\ge}(1-o(1))\mathbb{E}_{\x^T\sim \nu^T}\left[\inf_{\phi}
\mathbb{E}_{h\sim \mathcal{H}}\left[\sum_{t=1}^T\ell(\hat{y}_t,h(\x_t))
\right]\mid [T/\log^2 T]
\subset \x^T\right]-O(1) \\
&\overset{(d)}{\ge} (1-o(1))\inf_{\x^T\supset [T/\log^2T]}\inf_{\phi}
\mathbb{E}_{h\sim \mathcal{H}}\left[\sum_{t=1}^T\ell(\hat{y}_t,
h(\x_t))\right]-O(1)\\
&\overset{(e)}{\ge} (1-o(1))\inf_{\x^T\supset [T/\log^2T]}\inf_{Q}
\mathbb{E}_{y^T\sim P}[\log(1/Q(y^T))]-O(1)\\
&\overset{(f)}{\ge }(1-o(1)) \log|\mathcal{H}|-O(1)\ge (1-o(1))d\log(T/d)
\end{align*}
where $(a)$ follows since $y_t$ must equal to $h(\x_t)$ for all $t$ to maximize the 
regret~\footnote{This holds only for $h(\x_t)\in \{0,1\}$, while our upper bounds as in Section~\ref{sec-binary} hold for any \emph{binary} values in $[0,1]$.} (otherwise, the regret will be negative infinite since $h(\x_t)\in\{0,1\}$ 
and we may assume w.l.o.g. $\hat{y}_t\in [1/T,1-1/T]$ which only losses $O(1)$ on 
the optimal regret), $(b)$ follows by coupon collector 
and $\sup\ge \mathbb{E}$ where $h$ is choosing uniformly from 
$\mathcal{H}$, $(c)$ follows by $\inf\mathbb{E}\ge \mathbb{E}\inf$, 
$(d)$ follows by $\mathbb{E}\ge \inf$, $(e)$ follows by 
observing that $\{h(\x_t)\}_{t=1}^{T}$ 
is uniformly distributed over the label vectors of 
$\mathcal{H}$ restricted on the 
given $\x^T$ (we denote the distribution to be $P$) and 
$\sum_{t=1}^{T}\ell(\hat{y}_t,y_t)$ can be interpreted as minus $\log$ of some 
probability assignment $Q$ over $\{0,1\}^{T}$, $(f)$ follows by 
the fact that the $\inf_Q$ is achieved when $Q=P$ 
(by positivity of KL divergence) and the expectation expression 
equals the entropy of $P$ which further equals to 
$\log|\mathcal{H}|$ since $h$ is uniformly 
distributed over $\mathcal{H}$ and are distinct when restricted on 
$\x^T$ (since $[T/\log^2T]\subset \x^T$).

\section{Distributions with sufficient symmetry}
\label{app-symmetry}

One may observe that the main technique we used in Section~\ref{sec-binary} is to 
exploit the symmetry of $i.i.d.$ samples. In this appendix, we will show how 
such techniques can be generalized to distributions that have sufficient 
symmetry on the sample. We say a distribution $\mu$ over $\mathcal{X}^T$
is a product distribution with type $k$ if there exist distributions 
$\nu_1,\cdots,\nu_k$ over $\mathcal{X}$ such that
 $$\mu=\prod_{t=1}^T\nu_t,$$
where $\nu_t\in \{\nu_1,\cdots,\nu_k\}$.

We have the following analogous result of Theorem~\ref{th4}:

\begin{theorem}
\label{th10}
Let $\mathcal{H}$ be a binary valued class with finite VC-dimension, 
and $\mathcal{P}$ be the class of all product distributions over $\mathcal{X}^T$ with type $k$. 
Then there exists a global sequential covering set $\mathcal{G}$ of $\mathcal{H}$ 
at scale $\alpha=0$ and confidence $\delta$ such that:
$$\log |\mathcal{G}|\le O(k\VC (\mathcal{H})\log^2 T+\log T\log(1/\delta)).$$
Moreover, the linear dependency of $k$ can not be improved.
\end{theorem}

 \begin{proof}[Sketch of Proof]
For any product distribution $\mu$ of type $k$, we denote by $n_i$
the number of appearances of $\nu_i$ in the product of $\mu$. Let $\Phi$ be the
$1$-inclusion graph algorithm. For all $\x^T$, we consider the random permutation
$\sigma$ that permutes only the positions of 
$\x^T$ to positions that correspond to the same $\nu_i$. 
Applying Theorem~\ref{th3} (bounding the prediction error for each group), we know that for all $h\in\mathcal{H}$ the
expected number of errors of $\Phi$ on $h$ is upper bounded by:
$$\sum_{i=1}^k\VC (\mathcal{H})\log n_i\le k\VC (\mathcal{H})\log(T/k),$$
where the inequality follows by AM-GM inequality by noticing 
that $\sum_{i=1}^kn_i=T$.
The theorem follows by the same argument as in 
Lemma~\ref{lem1} to boost the expected
error bound to high probability error bound and applying Lemma~\ref{lem2}. 
The final part follows by considering
the linear threshold functions $1\{x\ge a\}$ and 
choosing $\nu_i$ to be singleton
distribution and selecting adversarially at the fist $k$ steps.
This ensures a lower bound of the covering set of size $2^k$.
\end{proof}
\begin{remark}
    Note that similar generalization can also be established for the results in Theorem~\ref{th5} and~\ref{th6}. We leave the investigation of more interesting distribution classes to future work.
\end{remark}

\section{Distribution non-blind regrets}
\label{app-nonblind}

We now discuss some results for the 
\emph{distribution non-blind} case. 
Let $\nu$ be an arbitrary distribution over $\mathcal{X}$ and 
$\mathcal{H}\subset \{0,1\}^{\mathcal{X}}$ be a binary valued class with 
finite VC-dimension. By classic result~\citep{haussler1995sphere} we know that for any $\epsilon>0$, there exists a class 
$\mathcal{F}_{\epsilon}\subset \mathcal{H}$ such that: 
$$
\forall h\in \mathcal{H}~\exists f\in \mathcal{F}_{\epsilon},~\mathrm{Pr}_{\x\sim 
\nu}[h(\x)\not=f(\x)]\le\epsilon,
$$
and
$$
|\mathcal{F}_{\epsilon}|\le \left(\frac{c}{\epsilon}\right)^{2\VC (\mathcal{H})},
$$
where $c>1$ is some absolute constant. A natural question is whether we can use 
$\mathcal{F}_{\epsilon}$ as a covering set and derive some meaningful bound 
for $\tilde{r}_T$. For clarity of presentation, we focus our 
attention on the logarithmic loss $\ell$. Suppose we apply the (truncated) Bayesian algorithm~\citep{wu2022precise} with truncation parameter $\frac{1}{T}$ over $\mathcal{F}_{\epsilon}$, we will achieve the following regret bound:
$$\tilde{r}_T(\mathcal{H},\{\nu^T\})=2\VC (\mathcal{H})\log(c/\epsilon)+\log T \cdot \mathbb{E}_{\x^T
\sim \nu^T}\left[\sup_{h\in \mathcal{H}}\inf_{f\in \mathcal{F}_{\epsilon}}
\sum_{t=1}^T1\{h(\x_t)\not=f(\x)\}\right].
$$

Clearly, the main difficulty is to estimate the quantity:
$$
\mathbb{E}_{\x^T\sim \nu^T}\left[\sup_{h\in \mathcal{H}}
\inf_{f\in \mathcal{F}_{\epsilon}}\sum_{t=1}^T1\{h(\x_t)\not=f(\x)\}\right].
$$
Note that, by the definition of $\mathcal{F}_{\epsilon}$, we have:
$$
\sup_{h\in \mathcal{H}}\inf_{f\in \mathcal{F}_{\epsilon}}~\mathbb{E}_{\x^T\sim 
\nu^T} \left[\sum_{t=1}^T1\{h(\x_t)\not=f(\x_t)\}\right]\le \epsilon T.$$

However, as we discussed in Section~\ref{sec-formulation}, 
moving the expectation outside of 
$\inf\sup$ is non-trivial. Nevertheless, we can prove the 
following somewhat surprising theorem~\footnote{We note that~\cite{haghtalab2022smoothed} establishes similar result for \emph{smooth adversary} samples but with a $d\log T$ bound.} that bounds the approximation error \emph{independent} of $T$ using a refined "double sampling trick":

\begin{theorem}
\label{th12}
Let $\mathcal{H}\subset \{0,1\}^{\mathcal{X}}$ be a binary valued class with finite VC-dimension, $\mathcal{F}_{\epsilon}$ is an $\epsilon$-cover of 
$\mathcal{H}$ under distribution $\nu$ over $\mathcal{X}$ as defined above. 
If $\epsilon=\frac{1}{T^2}$, then
$$
\mathbb{E}_{\x^T\sim \nu^T}\left[\sup_{h\in \mathcal{H}}\inf_{f\in \mathcal{F}_{\epsilon}}
\sum_{t=1}^T1\{h(\x_t)\not=f(\x_t)\}\right]\le 3\mathsf{VC}(\mathcal{H})+O(1).$$
Moreover, the linear dependency on $\mathsf{VC}(\mathcal{H})$ can not be improved even with $\epsilon=0$.
\end{theorem}

\begin{proof}
We now use a refined "double sampling trick" and denote $d=\mathsf{VC}(\mathcal{H})$ for simplicity. We pick two $i.i.d.$ samples $S_1$ and $S_2$ with size $T$ and $T^2$, respectively (note that $S_1,S_2$ have \emph{different} size). For any $h\in\mathcal{H}$, we denote $\hat{h}=\arg\min_{f\in \mathcal{F}_{\epsilon}}\mathrm{Pr}_{\x\sim \nu}[h(\x)\not=f(\x)]$. For any $k\le T$, we define events: $$A_1^k=\left\{\exists h\in \mathcal{H}~s.t.~\sum_{s\in S_1}1\{h(s)
\not=\hat{h}(s)\}\ge k\right\},
$$ 
and
$$
A_2^k=\left\{\exists h\in \mathcal{H}~s.t.~\sum_{s\in S_1}1\{h(s)\not=\hat{h}(s)\}\ge 
k~\text{ and }~\sum_{s\in S_2}1\{h(s)\not=\hat{h}(s)\}\le 2\right\}.
$$
By Markov inequality, we have $\mathrm{Pr}[A_2^k\mid A_1^k]\ge \frac{1}{2}$, i.e., $\mathrm{Pr}[A_1^k]\le 2\mathrm{Pr}[A_1^k\cap A_2^k]\le 2\mathrm{Pr}[A_2^k]$. It is therefore sufficient to upper bound $\mathrm{Pr}[A_2^k]$. By symmetries of $i.i.d.$ sample, we can \emph{fix} the samples $S_1,S_2$ and consider a random permutation over $S_1\cup S_2$. Note that for event of $A_2^k$ to happen, we must have at most $2$ elements $s\in S_1\cup S_2$ for which $h(s)\not=\hat{h}(s)$ switch to $S_2$. The probability of this happening is upper bounded by:
$$\frac{\binom{T^2}{2}\binom{T}{k}}{\binom{T^2+T}{k+2}}\le T^{4+k-2k+3}\le T^{7-k},\text{ for }T\ge 2,$$
where we used the standard binomial coefficient approximation and the fact that $k\le T$. Using a union bound on functions of $\mathcal{H}$ restricted on $S_1\cup S_2$, we have for $T\ge 2$ that $\mathrm{Pr}[A_2^k]\le T^{3d-k+7}.$ This implies $\mathrm{Pr}[A_1^k]\le \min\{1, 2T^{3d-k+7}\}$. We now observe that:
\begin{align*}
    \mathbb{E}_{\x^T\sim \nu^T}\left[\sup_{h\in \mathcal{H}}\inf_{f\in 
\mathcal{F}_{\epsilon}}\sum_{t=1}^T1\{h(\x_t)\not=f(\x_t)\}\right]&=
\sum_{k=0}^{\infty}\mathrm{Pr}\left[\sup_{h\in \mathcal{H}}
\inf_{f\in \mathcal{F}_{\epsilon}}\sum_{t=1}^T1\{h(\x_t)\not=f(\x_t)\}\ge k\right]\\
    &=\sum_{k=0}^T\mathrm{Pr}[A_1^k]\le \sum_{k=0}^{3d+7}\mathrm{Pr}[A_1^k]+\sum_{k=3d+8}^{T}\mathrm{Pr}[A_1^k]\\
    &\le  \sum_{k=0}^{3d+7} 1 + \sum_{k=3d+8}^{T}\mathrm{Pr}[A_1^k] \\
    &\le 3d+8+2\sum_{k=1}^{\infty}T^{-k}\le 3d+O(1).
\end{align*}
The lower bound on $d$ follows by selecting $\mathcal{H}\subset \{0,1\}^{[0,1]}$ to be functions taking value $1$ on at most $d$ positions, $\nu$ to be uniform distribution over $[0,1]$ and $\mathcal{F}_0$ to be the only all $0$ valued function.
\end{proof}

Taking $\epsilon=\frac{1}{T^2}$, we will get the following bound for logarithmic loss
$$(7\VC (\mathcal{H})+O(1))\log T.$$

Note that, this improves a $\log T$ factor compared to our general upper bound 
for distribution blind case given in Corollary~\ref{cor2}. We leave it as an 
open problem to determine if there is indeed a gap between distribution blind and 
non-blind setting. Note that~\cite{bhatt2021sequential} also achieves a similar 
bound for the distribution non-blind case. However, their bound is proved 
only for $\bar{r}_T$ (i.e., a bound of form $\sup_h\mathbb{E}$), which is 
much weaker than the bound we presented above for $\tilde{r}_T$ (i.e., a 
bound of form $\mathbb{E}\sup_h$).

\section{A $\Omega(\log T)$ lower bound of expected cumulative error}
\label{app-omega}

We now show that there exist classes $\mathcal{H}$ of finite VC-dimension such that for any prediction rule, the expected cumulative error with $i.i.d.$ sampling in the realizable case is lower bounded by $\Omega(\mathsf{VC}(\mathcal{H})\log T)$. This result was shown in~\cite[Section 3]{haussler1994predicting} and also in~\citep{antos1998strong}, however, we prove a simpler version here for completeness.

Let $\mathcal{H}=\{h_a(x)=1\{x\ge a\}:x,a\in [0,1]\}$ be the class of all 
linear threshold functions $[0,1]\rightarrow \{0,1\}$. We show that 
$\mathcal{H}$ is the desired class. To do so, we define 
$\nu$ to be the uniform distribution over $[0,1]$. Let $A,X_1,\cdots,X_T$ 
be $i.i.d.$ samples from $\nu$. We show that for any prediction rule $\Phi$ we have
$$
\mathbb{E}\left[\sum_{t=1}^T1\{\Phi(X_1^t,h_A(X_1^{t-1}))\not=h_A(X_t)\}\right]\ge 
\Omega(\log T).
$$

W.l.o.g., we can assume that $A,X_1,\cdots,X_T$ are distinct. By symmetries of 
$i.i.d.$, we may assume that $A, X_1,\cdots,X_T$ is a uniform random permutation 
for some fixed \emph{set} $\{A, X_1,\cdots,X_T\}$. We consider the following process: 
at the beginning Nature randomly permutes $\{A, X_1,\cdots,X_T\}$ and maintains
a set $S$ initially to be  equal to $\{X_1,\cdots,X_T\}$. At each time step $t$, 
Nature reveals the relative position of $X_t$ in $S$. After the predictor 
has made the prediction, Nature discards elements in $S$ on the opposite 
side of $X_t$ relative to $A$ (i.e., if $A>X_t$ we discard elements $\le X_t$ in $S$, else we discard elements $\ge X_t$) and reveals them to the predictor. Clearly, a 
lower bound for the game described above implies a lower bound for the 
original game, since the predictor gains more information at each time step. 
Denote $f(t)$ to be the expected number of errors the predictor will make if 
$|S|=t$. We have the following recursion:
$$
f(T)\ge \frac{2}{T}\sum_{t=1}^{T/2}\frac{t\cdot f(t)}{T}+\frac{(T-t) \cdot
f(T-t)}{T}+
\frac{t}{T}\ge \frac{2}{T}\sum_{t=1}^{T-1}\frac{t\cdot f(t)}{T}+\frac{t}{T}.$$
The reasoning goes as follows: conditioned on the event that $X_1$ is at 
position $t$, we have probability $t/T$ that $A$ is less than $X_1$ and 
probability $(T-t)/T$ that $A$ is larger than $X_1$. The best strategy for 
the predictor is to predict $A> X_1$ if $t\le T/2$ and predict $A<X_1$ otherwise. 
This contributes an expected error at first step to be $\frac{t}{T}$. 
The recursive formula now follows by the observation that conditioning on the 
position of $X_1$ and relative position of $A$ to $X_1$, the remaining set $S\cup \{A\}$ is still a uniform random permutation. We now claim that (the constant $0.01$ is not optimized):
$$f(t)\ge 0.01 \cdot \log t.$$
The base case for $t=1$ can be verified easily. 
We now prove by induction. We observe that
$$\int x \cdot \log(x)=\frac{1}{2}x^2 \cdot \log(x)-\frac{x^2}{4}.$$
Therefore, using Euler–Maclaurin formula we have:
$$
T^2 \cdot f(T)\ge 2\sum_{t=1}^{T-1}0.01 \cdot t\log(t)+t\ge 
0.01 \cdot T^2\log T-0.02 \cdot T\log T-
0.01\cdot T^2/4+T^2/2\ge 0.01 \cdot T^2\log T.$$
The result follows.

\end{document}